\documentclass{article}

 \usepackage[nonatbib, preprint]{neurips_2026}

\usepackage[utf8]{inputenc} 
\usepackage[T1]{fontenc}    
\usepackage{hyperref}       
\usepackage{url}            
\usepackage{booktabs}       
\usepackage{amsfonts}       
\usepackage{nicefrac}       
\usepackage{microtype}      
\usepackage{xcolor}         
\usepackage[toc,page,header]{appendix}
\usepackage{minitoc}

\usepackage{amsmath,amssymb,amsthm}
\usepackage{mathtools}
 \newcommand{\W}{W} \newcommand{\X}{\mathcal{X}}
\newcommand{\Y}{\mathcal{Y}} 

  \newcommand{\CC}{\mathbb{C}}
 \newcommand{\EE}{\mathbb{E}}

\newcommand{\PP}{\mathbb{P}}  \newcommand{\RR}{\mathbb{R}}

\newcommand{\Sobs}{\mathcal{O}}                 
\newcommand{\Smiss}{\mathcal{O}^{c}}            
\newcommand{\NA}{{\mathrm{NA}}}
\newcommand{\Bbias}{\mathcal{B}}                
\newcommand{\Aop}{\mathcal{A}}                  
\newcommand{\Rrem}{\mathcal{R}}                 
\theoremstyle{plain}
\newtheorem{proposition}{Proposition}
\newtheorem{corollary}{Corollary}
\newtheorem{remark}{Remark}
\newtheorem{theorem}{Theorem}[section]
\newtheorem{lemma}[theorem]{Lemma}
\theoremstyle{definition}

\theoremstyle{remark}

\usepackage{xcolor}
\usepackage[textwidth = 3cm]{todonotes}

\usepackage{graphicx}
\usepackage{float}
\usepackage{subcaption}
\usepackage{booktabs}
\usepackage{xcolor}
\usepackage{wrapfig}

\usepackage[numbers]{natbib}

\usepackage{array}

\title{Increasing Missingness to Reduce Bias: Richardson-SGD with Missing Data}

\author{
  \begin{tabular}{c}
    Ferdinand Genans\thanks{Corresponding author: \texttt{genans.ferdinand@gmail.com}}
    \qquad
    Erwan Scornet
    \\[0.5ex]
    \begin{tabular}{>{\centering\arraybackslash\normalfont}p{0.85\textwidth}}
      Sorbonne Université and Université Paris Cité, CNRS, Laboratoire de Probabilités, Statistique et Modélisation,
      LPSM, F-75005 Paris, France
    \end{tabular}
  \end{tabular}
}

\begin{document}

\maketitle

\begin{abstract}
    Stochastic gradient methods are central to modern large-scale learning, but their use with incomplete covariates remains delicate since imputation schemes generally introduce systematic gradient biases, as shown for linear models.
    In this work, we prove that all parametric models exhibit similar gradient bias for various imputation procedures and characterize exactly the dependence on the missingness ratio vector $p$, with $O(\|p\|)$ as the leading term.
    We exploit this analysis to propose a simple debiasing procedure for stochastic gradient descent (SGD) with missing values based on Richardson extrapolation, which leverages the exact expression of the gradient bias.
    The key idea is to \emph{deliberately add missingness}: from an already incomplete observation, we generate a further-thinned version at a higher, controlled missingness level, and combine the two resulting stochastic gradients to cancel the leading bias term. We prove that one Richardson step reduces the gradient bias from $O(\|p\|)$ to $O(\|p\|^2)$ under several missingness scenarios.
    Our proposed method is computationally efficient, model-agnostic and applies to any parametric loss whose stochastic gradient can be computed
    after imputation.
    Furthermore, when missing indicators are independent, the population gradient bias is a multilinear polynomial in $p$ and depends only on population gradient errors induced by declaring a single coordinate missing. In this case, our method generalizes to a multi-step Richardson procedure which recursively cancels higher-order terms.
    Empirically, Richardson debiasing improves optimization and estimation across several generalized linear models and combines positively with widely used imputation procedures such as MICE. These results suggest that, somewhat counter-intuitively, adding controlled missingness on top of existing missing data can make stochastic learning from incomplete data more accurate.
\end{abstract}

\section{Introduction}

Missing data are ubiquitous in modern machine learning. They may arise from database fusion, sensor failure, non-response in surveys, and selective acquisition pipelines, to name only a few. In his seminal paper, \citet{rubin1976inference} formalized the missing-data framework and introduced the now-standard taxonomy of three missingness regimes: \emph{Missing Completely at Random} (MCAR), in which missingness is independent of the data; \emph{Missing at Random} (MAR), in which missingness depends only on observed entries; and \emph{Missing Not at Random} (MNAR), in which missingness can also depend on the unobserved entries themselves.

\textbf{Framework.}
In supervised learning with incomplete covariates, one typically distinguishes two goals: estimating the parameters of a model despite the missing values, and producing a predictor with high test accuracy. The two are aligned when the test set is fully observed---accurate parameter estimation then leads to strong predictive performance---but they decouple when the test set itself contains missing entries, in which case a separate prediction-time strategy is required \citep[see e.g.][]{le2021sa,van2023missing,josse2024consistency}. We focus on the first objective and assume that missing values appear only in the training set, while the test set is complete. Even in this setting, parameter identifiability is not guaranteed under arbitrary MNAR mechanisms \citep[see, e.g., the examples and discussions in][]{robins1997toward,wang2014instrumental,miao2016identifiability}. We therefore restrict our attention to MCAR and a generalization---scalable MAR---in which the conditional missingness probability depends on a known intensity function.

\textbf{Handling missing data in parametric models.}
The simplest approach is \emph{complete-case} analysis \citep[][]{pigott2001review, little2019statistical}, which discards every sample containing at least one missing entry. This is unbiased under MCAR but throws away samples at a rate that is exponential in the dimension. The next-simplest approach is \emph{imputation}: missing entries are replaced by point estimates, after which any standard learning algorithm can be applied to the completed dataset. Constant imputation (zero or mean) is the most studied~\citep{Jones1996IndicatorAS} and the easiest to analyze, but it injects a systematic bias even under MCAR. Multivariate Imputation by Chained Equations~\citep[MICE,][]{van2011mice} and nearest-neighbour or neural-network imputation schemes \citep{troyanskaya2001missing, mattei2019miwae} reduce this bias empirically but offer few formal estimation guarantees.
A complementary line of work avoids imputation altogether by working with the joint distribution of inputs and mask. The Expectation--Maximization algorithm of \citet{dempster1977maximum}, refined for incomplete data by \citet{ibrahim1990incomplete} and extended to logistic regression via Stochastic Approximation EM~\citep{jiang2020logistic}, fits a parametric model to the inputs and the predictor jointly. These algorithms require a known parametric family for the covariates and can be expensive due to the E-step. A first review of estimation procedures for linear regression with missing covariates was given in \citet{little1992regression}.

\textbf{Related work - SGD with missing data.}
Gradient descent and its stochastic variants are the workhorse of large-scale learning, but they require a fully observed input to compute a gradient. The natural fix is to impute and then run SGD on the completed dataset; as observed by \citet{Jones1996IndicatorAS}, this leads to a biased gradient.
\citet{ayme2023naive,ayme2024random} study the test-time predictive performance of SGD applied to zero-imputed data when data can be missing in both train and test set, relating the imputation bias to a ridge regularization effect and leveraging the implicit bias of SGD~\citep{smithorigin} towards low-norm solutions to derive convergence rates \citep[see also][for missing data in high-dimensional linear models]{chandrasekher2020imputation, verchand2024high}. Another line of work focuses on parameter estimation or equivalently on test-time performance when the test set is assumed to contain complete data. \citet{loh2011high} characterize the exact bias induced by zero imputation in linear models and use this characterization to obtain the first parameter-estimation rates for sparse high-dimensional linear regression under MCAR. Building on their analysis, \citet{needell2019stochastic} design a stochastic gradient algorithm that is better suited to large-scale data, and \citet{sportisse2020debiasing} establish that averaged debiased SGD attains the optimal one-pass rate for linear regression.

\textbf{Contributions.}
We propose and analyze a debiasing procedure for stochastic gradients computed from imputed data, which can be applied to any parametric model and which is valid for a large class of imputation procedures. Our proposed method applies Richardson extrapolation~\citep{richardson1911ix} to the missingness scale $p$, yielding a model-agnostic correction that can be combined with a broad class of imputation rules. While Richardson extrapolation has been used in machine learning to remove leading-order biases in other contexts~\citep{bach2021effectiveness}, to the best of our knowledge, this is the first application to the missingness scale of a stochastic gradient. Our contributions are as follows.

\textit{Gradient-bias structure.} Under several MCAR and MAR settings described below,
we establish the exact expression of the population bias of a  stochastic gradient computed on imputed data. In doing so, we generalize the expression obtained by \citet{sportisse2020debiasing} for linear regression with zero-imputed data and independent MCAR missingness to any parametric model, a large class of imputation procedures, and a broad class of missingness scenarios (non-independent MCAR and MAR). As a consequence, we show that the gradient bias is $O(\|p\|)$, where $p=(p_1, \hdots, p_d)$ with $p_j$ being the probability that the $j$th component is missing. When the mask components are independent, we prove that the order of the remaining terms is $O(\|p\|^2)$.
Our bias decomposition holds for generic imputation procedures: better imputation may shrink the constants in $O(\|p\|)$ but cannot generally remove the leading $O(\|p\|)$ bias (Section~\ref{sec:bias-structure}).

\textit{Richardson-SGD.} We introduce a thinning construction that, from a sample with mask at scale $p$, generates a further-thinned mask at scale $Cp$, for some well-chosen $C >1$, using one extra Bernoulli draw per observed entry. A Richardson combination of the two gradients (computed on an imputed dataset at scales $Cp$) reduces the bias from $O(\|p\|)$ to $O(\|p\|^2)$ under some MCAR and MAR settings with independent masking components (Section~\ref{sec:richardson}). We also introduce a multi-step Richardson-based procedure which cancels higher-order terms, with exact cancellation for $d_{\mathrm{miss}}$ steps, where $d_{\mathrm{miss}}$ is the number of covariates subject to missingness.

\textit{Theory for one-pass SGD.} Our bias expansion plugs directly into classical biased-SGD proofs. For one-pass (one-epoch) SGD over $n$ samples, and given a smooth and strongly convex loss, Richardson-SGD attains $\EE\|w_n-w^\star\|^2=O(\|p\|^4)+O(1/n)$, against $O(\|p\|^2)+O(1/n)$ for plain imputed SGD. Multi-step Richardson reduces the missingness term further at the price of increased variance.
(Section~\ref{sec:rsgd}).

\textit{Experiments.} We validate the theory on synthetic and real datasets. For a variety of generalized linear models, Richardson-SGD improves over plain imputation under several  MCAR and MAR mechanisms and combines positively with MICE, Random-Forest MICE, and $k$-NN imputation (Section~\ref{sec:experiments}).

\section{Setting}
\label{sec:setting}

\paragraph{Random covariates and notation.}
Random variables are written in uppercase ($X$, $Y$, $M$); their realizations are written in the corresponding lowercase ($x$, $y$, $m$). For an integer $d\ge 1$ we set $[d]:=\{1,\dots,d\}$. For any $S\subseteq[d]$, we write $S^c:=[d]\setminus S$, and for a vector $v\in\RR^d$ we let $v^{(S)}\in\RR^{|S|}$ denote the subvector indexed by $S$.
Throughout the paper, $\|\cdot\|$ is the Euclidean norm and $\|\cdot\|_\infty$ the supremum norm.

\paragraph{Supervised learning and SGD.}
We consider a supervised learning setting with random covariates $X\in\X\subseteq\RR^d$, response $Y\in\Y$, parameter $w\in\W\subseteq\RR^q$, and a continuously differentiable loss $\ell:\W\times\X\times\Y\to\RR$. We aim at minimizing the population risk $L(w)\;:=\;\EE\left[\ell(w;X,Y)\right]$.
The complete-data single-sample gradient is $g(w;x,y):=\nabla_w\ell(w;x,y)$. Assuming differentiation and expectation commute, we have $\nabla L(w)=\EE[g(w;X,Y)]$. Because $\nabla L$ has no closed form in general, we use stochastic gradient descent (SGD), whose updates are given by
\begin{equation}
    w_{k+1}\;=\;w_k-\eta_k\hat g_k(w_k),
    \label{eq:sgd}
\end{equation}
where $\eta_k>0$ is the step-size and $\hat g_k(w_k)$ is a stochastic estimator of $\nabla L(w_k)$ computed at iteration $k$ from a sample or minibatch. When the sample is complete, $g(w;X,Y)$ is unbiased for $\nabla L(w)$.

\paragraph{Missing covariates and imputation.}
For each training sample, the learner observes a realisation of $(X^{\mathrm{obs}},Y,M)$, where $M\in\{0,1\}^d$ is a missingness mask and, for all $j \in [d]$, $X^{\mathrm{obs}}_j = X_j$ if $M_j=0$ and $X^{\mathrm{obs}}_j = \NA$ if $M_j=1$.
We write $p_j:=\PP(M_j=1)$ for the marginal missingness probability of feature $j$ and $p:=(p_1,\dots,p_d)$ for the missingness vector. Missing entries are filled in by an imputation rule $\mathcal{I}$ that produces an imputed covariate vector
\begin{align}
    \tilde X\;:=\;\mathcal{I}\big(X^{\mathrm{obs}},M,\xi\big), \label{def:imputation}
\end{align}
where $\xi \perp M\mid (X,Y)$ collects auxiliary randomness used by $\mathcal{I}$. We focus on \textit{data-independent imputation rules} $\mathcal{I}$ that impute each observation independently of the others. This assumption makes our analysis tractable by enabling a decomposition at the sample level. Standard imputations (mean, iterative) can be slightly modified to fall into this setting by training $\mathcal{I}$ on an auxiliary dataset.
The imputed stochastic gradient available to the learner is $\hat g(w)\;:=\;g\big(w;\tilde X,Y\big)$, which is, in general, a \emph{biased} estimator of $\nabla L(w)$, with bias
\begin{equation}
    \Bbias(w,p)\;:=\;\EE\left[\hat g(w)\right]-\nabla L(w),
    \label{eq:bias-def}
\end{equation}
the expectation being over $(X,Y)$, the mask $M$, and the imputation randomness $\xi$.

\paragraph{Missingness mechanisms.}
We follow the taxonomy of \citet{rubin1976inference}: the mask is \emph{Missing Completely at Random} (MCAR) when $M\perp(X,Y)$, and \emph{Missing at Random} (MAR) when, conditionally on the observed entries, $M$ is independent of the missing entries.
Throughout the paper, we let $\Sobs\subseteq[d]$ (possibly empty) be the set of indices of variables that are \emph{always-observed}. We let $V:=X^{(\Sobs)}$ be the vector of \emph{always-observed variables}.
To enable a tractable analysis, we focus on two concrete mechanisms, which depend on the probability vector $p$,  assumed to be known.

\textbf{Heterogeneous MCAR (hMCAR).} $M$ is independent of $(X,Y)$ and $\PP(M_j=1)=p_j$.

\textbf{Scalable MAR (sMAR).} $\{M_j\}_{j\in\Smiss} \perp (X^{(\Sobs^c)}, Y)\mid V$, and for every $j\in\Smiss$,
\[
    \PP\big(M_j=1\,\big|\,V\big)\;=\;p_j\,q_j(V),
\]
for known intensity functions $q_j:\RR^{|\Sobs|}\to[0,p_j^{-1}]$ with $\EE[q_j(V)]=1$.

It is known that MAR settings contain scenarios of different difficulties \citep{naf2024good}, some of which being close to MNAR settings \citep{molenberghs2008every}, for which identifiability does not always hold \citep[see, e.g.,][]{robins1997toward,wang2014instrumental,miao2016identifiability}. Thus, we restrict the MAR settings we consider via the sMAR assumption. Note that the condition $\mathbb{E}[q_j(V)]=1$ in sMAR is necessary to ensure that $p_j = \mathbb{P}(M_j=1)$.
A concrete example of sMAR is a logistic missingness mechanism, as commonly used in simulation studies of missing covariates~\citep[e.g.][]{marshall2010comparison,wang2023score}. We say a mask is \emph{independent hMCAR} (resp.\ \emph{independent sMAR}) if it is hMCAR (resp.\ sMAR) and the $\{M_j\}_{j\in\Smiss}$ are mutually independent (resp.\ conditionally on $V$).

Our objective remains the complete-data risk $L(w)$ and its minimizer $w^\star$; missingness and imputation only affect the stochastic gradients used to optimize it. Our goal is to replace the imputed gradient $\hat g$ in~\eqref{eq:sgd} by a corrected gradient $\hat g^{\mathrm{R}}$, computed from the same observation plus a small amount of controlled additional thinning, so as to cancel or shrink the gradient bias~\eqref{eq:bias-def}.

\section{First-order structure of the missingness bias}
\label{sec:bias-structure}

Before designing a debiasing procedure, we describe the structure of the imputation-induced gradient bias as a function of the missingness scale $p$. The key observation is that, regardless of the loss and the imputation rule, the bias admits a clean expansion whose leading term is \emph{linear} in $p$ and whose coefficients are population gradient gaps that do not depend on $p$. This expansion will be the structural fact that Richardson extrapolation later exploits.

\begin{proposition}[First-order structure of the missingness bias]
    \label{prop:grad_bias_fo_struct}
    Consider any data-independent imputation defined in \eqref{def:imputation} and assume hMCAR or sMAR holds. Then the population gradient bias~\eqref{eq:bias-def} can be decomposed as
    \begin{equation}
        \Bbias(w,p)
        \;=\;
        \Aop(w)\,p
        \;+\;
        \Rrem(w,p),
        \label{eq:bias-decomposition}
    \end{equation}
    with $\Aop(w)\in\RR^{q\times d}$ independent of $p$. Letting  $a_j(V) = 1$ for hMCAR and $ a_j(V) = q_j(V)$ for sMAR, the $j$-th column of $\Aop(w)$ is the population gradient gap obtained by declaring coordinate $j$ missing:
    \begin{equation}
        \Aop_{\cdot j}(w)
        \;=\;
        \EE\left[
            a_j(V)\,
            \Big\{
            G_{\{j\}}(w;X,Y,\xi)
            -
            g(w;X,Y)
            \Big\}
            \right]\ .
        \label{eq:bias-Aj}
    \end{equation}
\end{proposition}

The remainder $\Rrem(w,p)$ contains the co-missingness contributions, namely the terms involving simultaneous missingness of two or more coordinates. The exact expression of $\Rrem(w,p)$ is given in Appendix~\ref{app:bias-decomposition}. The proof is based on a discrete-difference expansion over missingness patterns and separates the contribution of each joint missingness pattern $S\subseteq[d]$.

In full generality, the remainder is at most linear in $p$, while it is $o(\|p\|)$ in most scenarios. Indeed, strong dependence among mask components can make co-missingness terms contribute at first order. For instance, this may occur when two coordinates are perfectly negatively associated, so that $\PP(M_j=1\mid M_k=1)=0$. The following corollary identifies a key regime motivating Richardson extrapolation.

\begin{corollary}\label{cor:bias_asymp_indep}
    Under the assumptions of Proposition~\ref{prop:grad_bias_fo_struct}, suppose in addition that the missingness indicators $\{M_j\}_{j\in\Smiss}$ are conditionally independent given $V$. Then
    \begin{equation}
        \|\Rrem(w,p)\|
        \;=\;
        O(\|p\|^2),
        \qquad\text{and therefore}\qquad
        \big\|\Bbias(w,p)-\Aop(w)\,p\big\|
        \;=\;
        O(\|p\|^2).
        \label{eq:bias-fo-bound}
    \end{equation}
\end{corollary}

Under independent hMCAR$\backslash$sMAR, the bias is a \emph{multilinear} in $p$ (see proof of Corollary \ref{cor:bias_asymp_indep}),
\begin{equation}
    \Bbias(w,p)
    =
    \sum_{\varnothing\ne S\subseteq[d]} \mu_S(w) \,
    \Big(\prod_{j\in S}p_j\Big),
    \quad
    \mu_S(w)
    \;:=\;
    \EE\left[\Big(\prod_{j\in S}a_j(V)\Big)\,\Delta_S G_\varnothing(w;X,Y,\xi)\right],
    \label{eq:bias-poly}
\end{equation}
where $T_jG_\varnothing:=G_{\{j\}}$ declares coordinate $j$ missing and $\Delta_S G_\varnothing:=\prod_{j\in S}(T_j-I) G_\varnothing$. The coefficient $\mu_S(w)$ aggregates the effect of $|S|$-fold co-missingness. Equation~\eqref{eq:bias-poly} is the structural fact that drives both first and higher-order Richardson cancellation.

The decomposition has three implications. \emph{(i) The leading bias is linear in $p$:} under conditional independence, the remainder is $O(\|p\|^2)$, so the first-order behavior is fully captured by $\Aop(w)\,p$. \emph{(ii) The leading operator is an average gradient gap:} the column $\Aop_{\cdot j}(w)$ vanishes whenever coordinate $j$ is always observed or is perfectly recovered by $\mathcal{I}$. \emph{(iii) Imputation reduces constants, not the leading order:} the expansion holds for any data-independent imputation, and a better imputation rule only shrinks the entries of $\Aop_{\cdot j}(w)$ without changing the order of $\Bbias(w,p)$ in $p$.

These three points together suggest a clear strategy. Imputation alone cannot remove the leading $O(\|p\|)$ scaling, except if it fully recovers the covariate. Improving the imputation only refines the constants $\Aop_{\cdot j}(w)$. To eliminate the leading order, we propose to act \emph{on $p$ itself}---that is, evaluate the imputed gradient at two different missingness scales and combine the results so that the linear contribution cancels. This is precisely what Richardson extrapolation achieves, and the construction we develop in the next section turns this idea into a practical SGD update.

\section{Richardson-SGD}
\label{sec:richardson}

\paragraph{Richardson extrapolation in a nutshell.}
Richardson extrapolation~\citep{richardson1911ix} cancels the leading term of an asymptotic expansion. If $T(p)=T_0+pT_1+p^2T_2+o(p^2)$ as $p\to0$ and $C>1$, the combination
\begin{equation}
    T_C^{\mathrm{R}}(p)\;:=\;\tfrac{C\,T(p)-T(Cp)}{C-1}
    \;=\;T_0-Cp^2T_2+o(p^2)
    \label{eq:richardson-scalar}
\end{equation}
eliminates the linear term. With $k+1$ scales $1=C_0<C_1<\cdots<C_k$ and a Vandermonde weight vector, the first $k$ orders are cancelled simultaneously~\citep{pages2007multi}.

At first sight, applying~\eqref{eq:richardson-scalar} to the missingness bias would require evaluating the imputed gradient at \emph{two} missingness scales $p$ and $Cp$ on the same observation. The learner, however, only observes a single mask $M^{(p)}$ at scale $p$. We resolve this with a single extra Bernoulli draw per observed entry: from a sample at scale $p$, we \emph{further thin} it to obtain a mask whose conditional law given $X$ is exactly that of an independent  draw at scale $Cp$. No new observation is required. We employ this additional mask to propose a Richardson-corrected gradient, used in lieu of the standard gradient in a SGD procedure.

\paragraph{Further-thinned mask.}
Fix $C>1$ such that, for all  $j\in\Smiss$, $C\,p_ja_j(V)\le 1$ almost surely. Conditional on $(X,M^{(p)})$, draw independent thinning bits $r_j$ with $r_j=1$ for $j\in\Sobs$ and, for $j\in\Smiss$,
\begin{equation}
    r_j\,\big|\,(X,M^{(p)})\;\sim\;
    \mathrm{Bernoulli}\!\left(\frac{1-Cp_ja_j(V)}{1-p_ja_j(V)}\right),
    \textrm{and let} \quad
    M_j^{(Cp)}\;:=\;1-(1-M_j^{(p)})\,r_j.
    \label{eq:thinning}
\end{equation}
All entries missing under $M^{(p)}$ stay missing under $M^{(Cp)}$; an observed entry is hidden under $M^{(Cp)}$ exactly when $r_j=0$. A short calculation (Appendix~\ref{app:richardson-proofs}) gives $\PP(M_j^{(Cp)}=1\mid V)=Cp_ja_j(V)$, so $M^{(Cp)}$ has the same conditional law as the original mask but at scale $Cp$.

\paragraph{Richardson-corrected gradient.}
Equipped with the further-thinned mask, we can apply~\eqref{eq:richardson-scalar} to the imputed gradient. Crucially, we must \emph{not} impute the same observation at two different missingness levels, since we need common missing values between the two scales to be identical (see Appendix~\ref{app:linked-imputation} for further explanation and a numerical illustration). We impute once on the more thinned sample at scale $Cp$, then \emph{restore} the artificially hidden entries to recover the imputation at scale $p$:
\begin{equation*}
    \tilde X^{(Cp)}\;:=\;\mathcal{I}\big(X^{\mathrm{obs}},M^{(Cp)},\xi\big),
    \qquad
    \tilde X_{Cp,j}^{(p)}
    :=
    \begin{cases}
        X_j,               & M_j^{(p)}=0, \\
        \tilde X_j^{(Cp)}, & M_j^{(p)}=1,
    \end{cases}
    \quad \textrm{for all } j \in [d]\ .
\end{equation*}
Set $\hat g^{(p)}(w):=g(w;\tilde X^{(p)},Y)$ and $\hat g^{(Cp)}(w):=g(w;\tilde X^{(Cp)},Y)$. The \emph{Richardson-corrected gradient} is
\begin{equation}
    \hat g_C^{\mathrm{R}}(w)\;:=\;\frac{C\,\hat g^{(p)}(w)\;-\;\hat g^{(Cp)}(w)}{C-1}.
    \label{eq:richardson-grad}
\end{equation}

\textbf{Richardson-SGD}
plugs $\hat g_C^{\mathrm{R}}$ into the SGD update~\eqref{eq:sgd}. For each sampled observation $(x_i^{\mathrm{obs}},m_i,y_i)$ at iteration $k$:
\begin{enumerate}
    \itemsep -2pt
    \item \emph{Original masked sample.} Read off the mask $m_i^{(p)}$ at scale $p$.
    \item \emph{Further-thinned sample.} Draw $r$ as in~\eqref{eq:thinning}. For all $j \in [d]$, $m_{ij}^{(Cp)} \gets 1-(1-m_{ij}^{(p)})\,r_{j}$.
    \item \emph{One imputation.} Compute $\tilde x_i^{(Cp)}\gets\mathcal{I}(x_i^{\mathrm{obs}},m_i^{(Cp)},\xi_i)$, then obtain $\tilde x_i^{(p)}$ by overwriting the entries hidden by $r$ with their true values from $x_i$.
    \item \emph{Gradient estimates.} Evaluate $\hat g_i^{(p)}\!:=g(w_k;\tilde x_i^{(p)},y_i)$ and $\hat g_i^{(Cp)}\!:=g(w_k;\tilde x_i^{(Cp)},y_i)$.
    \item \emph{Richardson correction \& SGD update.} Form $\hat g_k^{\mathrm{R}}\gets(C\,\hat g_i^{(p)}-\hat g_i^{(Cp)})/(C-1)$ and update $w_{k+1}\gets w_k-\eta_k\,\hat g_k^{\mathrm{R}}$ (averaged across a minibatch when $b>1$).
\end{enumerate}
The procedure is a thin wrapper around any imputation-based SGD pipeline: one extra Bernoulli draw per observed entry and one extra gradient evaluation per sample.

\section{Theory of Richardson-SGD}
\label{sec:rsgd}

We now state the theoretical guarantees of Richardson-SGD. The analysis shows that Richardson corrections successively cancel the terms in the bias expansion, while controlling the associated variance inflation and the error from estimating the missingness mechanism. Combining these bounds with a classical biased-SGD argument yields a convergence rate. Throughout, the result applies to \emph{one-pass (one-epoch) SGD}, as in~\citet{sportisse2020debiasing} for linear regression: each sample is visited once, and the bias expansion from Section~\ref{sec:bias-structure} feeds directly into standard biased-SGD arguments. Multi-epoch behavior is outside the scope of the theory and is examined empirically in Section~\ref{sec:experiments}.

\subsection{First-order bias cancellation}

\begin{proposition}[First-order debiasing]
    \label{prop:first-order-debias}
    Assume independent hMCAR or independent sMAR. Then
    \begin{equation}
        \big\|\,\EE\big[\hat g_C^{\mathrm{R}}(w)\big]-\nabla L(w)\,\big\|
        \;=\;O(\|p\|^2),\qquad \textrm{when } \|p\|\to 0.
        \label{eq:first-order-cancel}
    \end{equation}
\end{proposition}

Proposition~\ref{prop:first-order-debias} shows that the debiasing challenge can be met by a deliberately counterintuitive operation: we decrease bias by adding missing values. While the plain imputed gradient has bias of order $\|p\|$, the Richardson-corrected gradient constructed from the original and further-thinned masks cancels this leading term and leaves only an $O(\|p\|^2)$ bias under independent hMCAR or independent sMAR. This gain is uniform in the loss and the imputation rule, and requires only one additional Bernoulli draw and one additional gradient evaluation per sample (proof in Appendix~\ref{app:richardson-proofs}).

\subsection{Higher-order Richardson-SGD under independent masks}

When the missing indicators are conditionally independent given $V$, Section~\ref{sec:bias-structure} showed that the gradient bias is, in fact, a multilinear polynomial in $p$. Since Richardson extrapolation is itself linear in the underlying expansion, one can cancel further orders by combining estimators at more than two missingness scales. Iterating the thinning construction with $k+1$ scales $1=C_0<C_1<\cdots<C_k$ (cascaded via~\eqref{eq:thinning} with $C\gets C_\ell/C_{\ell-1}$) and Vandermonde weights $\alpha\in\RR^{k+1}$ yields the $k$-th order Richardson estimator $\hat g^{[k]}(w):=\sum_{\ell=0}^{k}\alpha_\ell\,\hat g^{(C_\ell p)}(w)$.

\begin{corollary}[Higher-order cancellation]
    \label{cor:k-order}
    Assume $C\,p_j\,a_j(V)\le 1$ for every $j$. Under independent hMCAR or independent sMAR, $\big\|\EE[\hat g^{[k]}(w)]-\nabla L(w)\big\|=O(\|p\|^{k+1})$ as $\|p\|\to 0$. Furthermore, with $d_{\mathrm{miss}}:=\textrm{Card}(\{j:p_j>0\})$, the $d_{\mathrm{miss}}$-th order estimator cancels the bias exactly: $\EE[\hat g^{[d_{\mathrm{miss}}]}(w)]=\nabla L(w)$.
\end{corollary}

For linear regression with squared loss, the bias is a polynomial of degree at most $2$ in $p$ (Appendix~\ref{app:linreg}), so the two-step Richardson-SGD produces an \emph{exact} debiasing under both hMCAR and sMAR. This matches the closed-form correction mechanism of \citet{sportisse2020debiasing} as a special case and extends it to sMAR, where no closed form is available. More generally, Corollary~\ref{cor:k-order} suggests that higher-order Richardson-SGD should be most useful when only a few coordinates are subject to missingness ($d_{\mathrm{miss}}$ small) so that the corresponding polynomial degree is low, or that the highest polynomial degree in the bias
is low, as for linear regression (see Appendix \ref{app:glm_bias_formulas}). Figure~\ref{fig:sgd-rich-multi-order} illustrates this phenomenon in synthetic linear and logistic regressions.

\begin{figure}[h]
    \centering
    \includegraphics[width=0.95\linewidth]{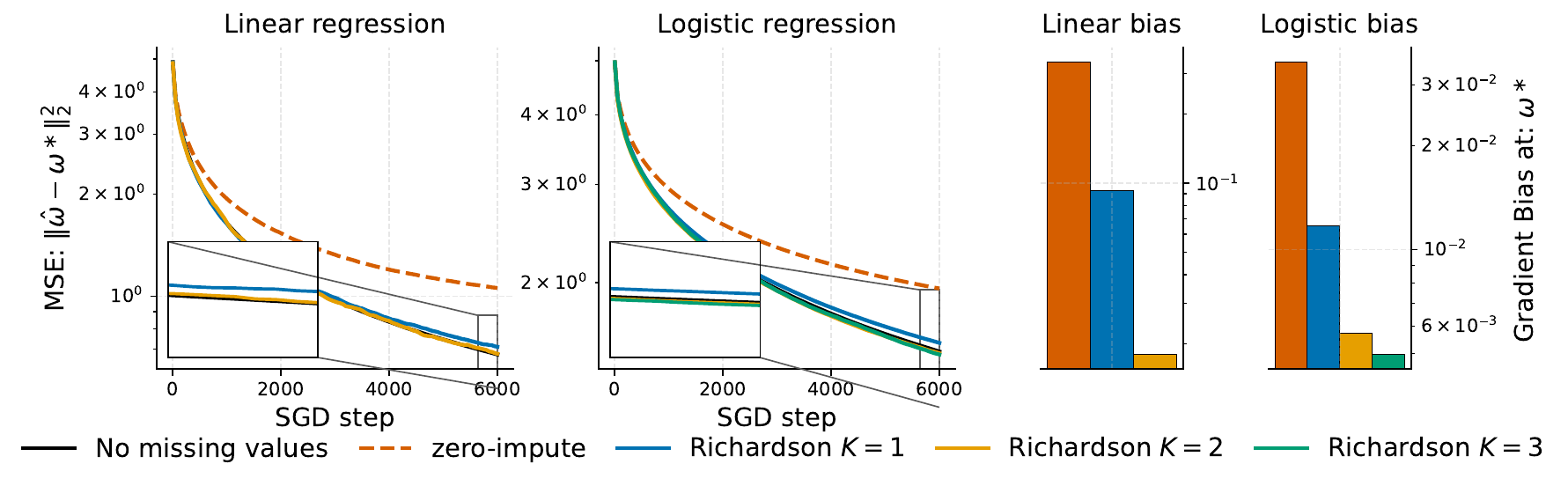}
    \caption{\small Multi-order Richardson correction in $4$-covariate linear and logistic regression under hMCAR with $p=(0.10,0.15,0.08,0.12)$, and bias at $w^\star$, $\|\EE[\hat g(w^\star)]-\nabla L(w^\star)\|$ As predicted by Corollary~\ref{cor:k-order}, each Richardson level removes one further order of bias: first-order Richardson clearly improves over standard SGD on zero-imputed data, second-order matches the complete-data trajectory in linear regression (curves overlap), and third-order yields an additional gain in the logistic case for the bias.}
    \label{fig:sgd-rich-multi-order}
\end{figure}

\subsection{Variance inflation}
\label{subsec:variance}

The previous two subsections highlight how Richardson reduces bias. As is standard in extrapolation methods, this comes at a price: the corrected gradient is a difference of two estimators evaluated at different missingness levels, which inflates its variance. Quantifying this inflation is essential, since the convergence rate of SGD depends on \emph{both} the bias and the variance of the stochastic gradient. For the first-order estimator,
\begin{equation}
    \mathrm{Var}\big[\hat g_C^{\mathrm{R}}(w)\big]
    \;\le\;
    \frac{2\,\big(C^2\,\mathrm{Var}\big[\hat g^{(p)}(w)\big]\;+\;\mathrm{Var}\big[\hat g^{(Cp)}(w)\big]\big)}{(C-1)^2},
    \label{eq:variance-bound}
\end{equation}
with a larger $C$ controlling the multiplicative factor since the function $f:(1,+\infty)\to\RR$, $f(x)=x^2/(x-1)^2$, is decreasing. For the $k$-th order estimator, variance inflates by a factor that grows with $k$, and requires $k$ missingness upscales, which limits $k$ when some $p_j$ are large. We therefore use first-order Richardson by default and reserve higher-order constructions for small $d_{\mathrm{miss}}$ or for losses with low maximum polynomial degree, as linear regression, which is of degree $2$ (see Appendix \ref{app:glm_bias_formulas}).

\subsection{Richardson-SGD with estimated missingness parameters}
\label{subsec:plug-in}

So far we have assumed that the quantities $(p,q)$ driving the missing mechanism are known. In practice, $p_j$ is estimated by the empirical missingness frequency on coordinate $j$, while $q_j(V)$ is fitted by a probabilistic model with input $V$. We now quantify how the resulting estimation errors propagate into the Richardson bias. Using the identifiability convention $\EE[q_j(V)]=1$, let $\lambda_j(V):=p_jq_j(V)$ and $\hat\lambda_j(V):=\hat p_j\hat q_j(V)$. The plug-in thinning rule replaces~\eqref{eq:thinning} by $\tilde r_j\sim\mathrm{Bernoulli}((1-C\hat\lambda_j(V))/(1-\hat\lambda_j(V)))$ and yields the plug-in Richardson gradient $\hat g_{C,\hat\lambda}^{\mathrm{R}}$.

\begin{proposition}[Plug-in Richardson]
    \label{prop:plug-in}
    Assume hMCAR or sMAR with $\lambda_j(V),\hat\lambda_j(V)\le\rho<1$, $C\hat\lambda_j(V)\le 1$ for every $j$, and $\|G_S(w;X,Y,\xi)-\nabla L(w)\|_{L^2}\le G_\star$ for every $S\subseteq[d]$. If $\|\hat p-p\|_\infty\le\delta_p$ and $\max_j\sup_v|\hat q_j(v)-q_j(v)|\le\delta_q$, then
    \begin{equation}
        \big\|\,\EE[\hat g_{C,\hat\lambda}^{\mathrm{R}}(w)]-\nabla L(w)\,\big\|
        \;=\;O\!\left(\|p\|^2+\delta_p+\|p\|_\infty\delta_q+\delta_p\delta_q\right).
        \label{eq:plug-in-bound}
    \end{equation}
    Under hMCAR ($q_j\equiv 1$, $\delta_q=0$), this collapses to $O(\|p\|^2+\delta_p)$.
\end{proposition}

The leading $\Aop(w)\,p$ contribution is cancelled regardless of plug-in errors, up to an additive $O(\delta_p+\|p\|_\infty\delta_q)$ penalty (proof in Appendix~\ref{app:plug-in}). When $\delta_p,\delta_q$ shrink fast enough, the $O(\|p\|^2)$ term dominates and the exact-mechanism guarantee is recovered. This behavior further motivates using the first-order Richardson SGD scheme, while higher order might not be conclusive in the plug-in setting. Appendix~\ref{app:exp-robust} reports an empirical sensitivity study.

\subsection{One-pass SGD convergence}
\label{subsec:one-pass}

We have now controlled both the bias of the Richardson-corrected gradient, through Proposition~\ref{prop:first-order-debias} and Corollary~\ref{cor:k-order}, and its variance, through \eqref{eq:variance-bound}, including under plug-in mechanisms (Proposition~\ref{prop:plug-in}). It remains to translate these gradient-level guarantees into a convergence rate for the SGD iterates, which is the quantity of interest. Note that biased SGD schemes have been extensively studied in the literature \citep[see, e.g.][]{ajalloeian2020convergence, demidovich2023guide}. To illustrate the resulting bias improvement of Richardson-SGD compared to plain imputation, we give a result under classic regularity conditions on the loss function.

\begin{corollary}[One-pass Richardson-SGD]
    \label{cor:one-pass}
    Assume $L$ is $\alpha$-strongly convex and $\beta$-smooth, the per-sample stochastic gradients are bounded in $L^2$. Under independent hMCAR or independent sMAR, after one pass on $n$ i.i.d. samples with $\eta_k = \frac{c}{k+\gamma},  c>\frac{1}{\alpha},\gamma \ge \frac{6c\beta^2}{\alpha} $, we obtain
    \begin{equation}
        \EE\big\| w_n-w^\star\big\|^2
        \;=\;
        \begin{cases}
            O(\|p\|^2)+O(1/n) & \text{(plain imputed SGD)},           \\
            O(\|p\|^4)+O(1/n) & \text{(Richardson-SGD, first order)},
        \end{cases}
        \label{eq:one-pass-rates}
    \end{equation}
    and the same convergence orders hold for the excess test loss $\EE[L(w_n)-L(w^\star)]$. With $k$-step Richardson-SGD, the missingness floor becomes $O(\|p\|^{2(k+1)})$. Thus,  for sufficiently large $k$, the missingness contribution is dominated by the statistical floor $O(1/n)$.
\end{corollary}

Two implications of Corollary~\ref{cor:one-pass} are worth highlighting.
First, when $\| p\| \ll 1$, one-step Richardson-SGD improves the bias floor of plain imputed SGD from $O(\| p\|^2)$ to $O(\| p\|^4)$. Thus, the missingness-induced contribution is reduced by two orders of magnitude in $\| p\|$, while keeping essentially the same per-iteration cost.
Second, multi-step Richardson-SGD can, in principle, reduce the missingness term down to the statistical noise level $O(1/n)$. This comes at the price of variance inflation: the bound in~\eqref{eq:variance-bound} compounds across Richardson levels and may become prohibitive when $d_{\mathrm{miss}}$ is large or when the loss has heavy stochastic gradients. Consequently, multi-step Richardson-SGD is most appealing when $d_{\mathrm{miss}}$ is small, or in settings such as linear regression where order~$2$ already suffices.

\textbf{Scope of the theory.} We emphasize that Corollary~\ref{cor:one-pass} is a one-pass guarantee, in line with the regime studied by~\citet{sportisse2020debiasing}. The multi-epoch behavior is not covered by our analysis: when iterates revisit the same observations, the gradient noise due to missing values across iterations are no longer independent. The experiments of Section~\ref{sec:experiments} suggest, however, that Richardson-SGD remains effective in multi-epoch training, and we view a formal multi-epoch analysis as an interesting question for future work.

\section{Experiments}
\label{sec:experiments}

We empirically study Richardson-SGD on synthetic and real datasets available in \texttt{scikit-learn} \citep{pedregosa2011scikit}. Throughout, missing entries are introduced \emph{ex post} into otherwise complete datasets according to the mechanism specified in each subsection, either hMCAR or sMAR, so that the ground truth $w^\star$ is known, or can be estimated by multi-pass training with L-BFGS-B, and the quantities $p$ and $q_j$ are also known. Unless stated otherwise, the average missingness is fixed at $\bar p = 0.2$. To keep the main text concise, we report only logistic regression here; analogous experiments for other datasets and models, including linear and Poisson GLMs, together with implementation details, are deferred to Appendix~\ref{app:additional-experiments}.

\textbf{Empirical takeaway.}
Across datasets, models, missingness mechanisms, and imputation rules, Richardson-SGD behaves as a generic debiasing layer rather than a model-specific correction. It improves imputation-based SGD using only controlled thinning and one additional gradient evaluation, and remains effective when the missingness mechanism is estimated or partially misspecified. In short, the method is simple, fast, model-agnostic, and theoretically grounded, making it a natural add-on for learning with missing covariates.

\subsection{Richardson with imputation on logistic regression}
\label{subsec:exp-imputation}

This experiment tests the central practical claim of the paper: \emph{Richardson extrapolation can be combined effectively with standard imputation methods}. We run logistic regression under hMCAR missingness, comparing SGD applied on the most standard imputations (namely MICE, MICE with random-forest base learners, and $k$-nearest-neighbor imputations) used in conjunction with SGD, and the Richardson-SGD counterparts (applied to the same imputation procedures). Across missingness levels and datasets, Richardson consistently acts as a complementary debiasing layer: the imputer reduces the initial missingness bias, while Richardson further reduces the residual gradient bias, with the largest gains obtained when the underlying imputer is already accurate.
\begin{figure}[h]
    \centering
    \includegraphics[width=\textwidth,height=0.18\textheight]{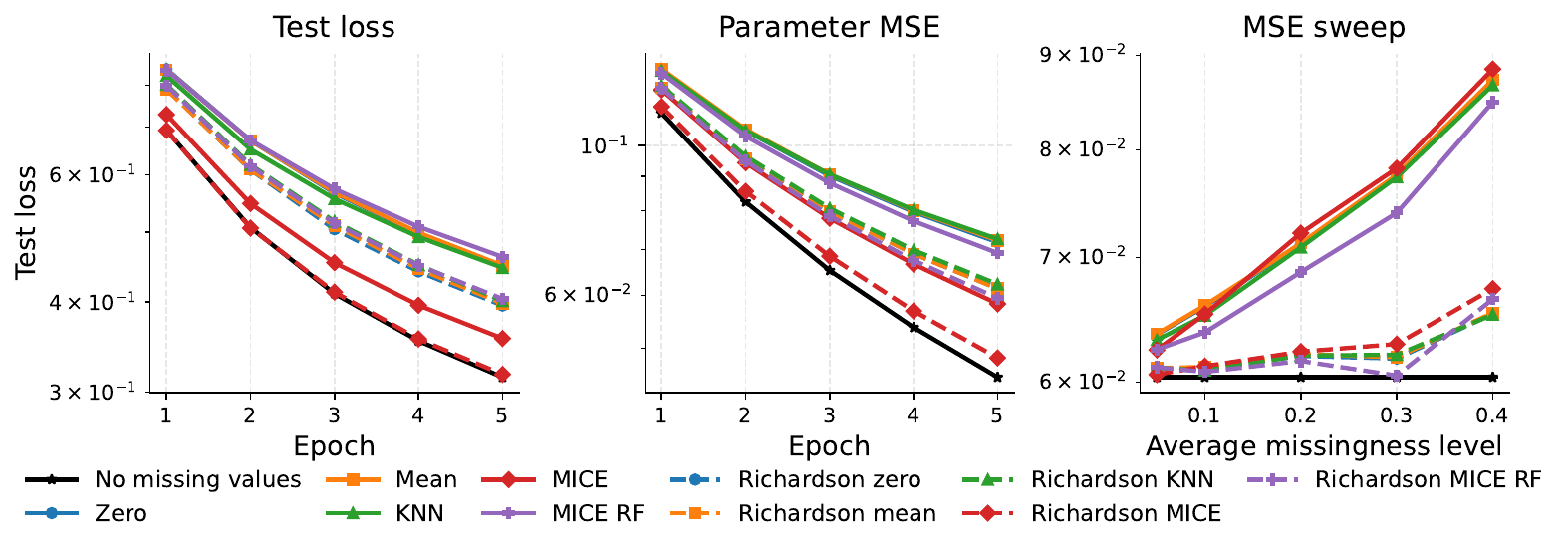}
    \caption{
        \textbf{Covertype missingness sweep.}
        Test loss and parameter mean-squared error for logistic regression on Covertype~\citep{covertype_31} under hMCAR missingness, as the average missingness level $\bar p$ varies.
        Richardson improves over each corresponding imputed SGD baseline across a broad range of $\bar p$, showing that the correction is not limited to the very small-missingness regime.
    }
    \label{fig:logsweep}
\end{figure}

\begin{figure}[h]
    \centering
    \includegraphics[width=1\linewidth]{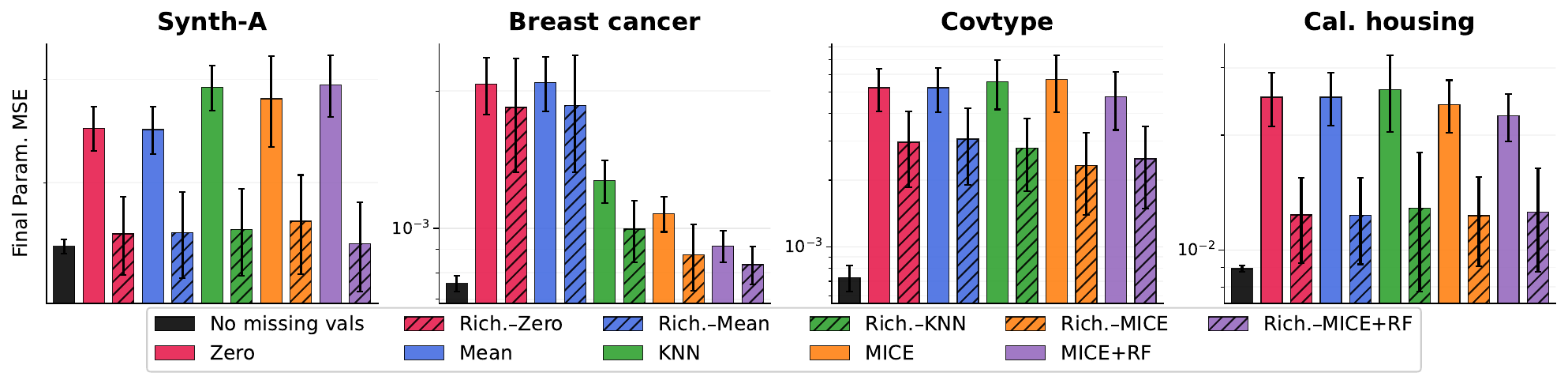}
    \caption{
        \textbf{Cross-dataset comparison.}
        Final parameter mean-squared error of the last SGD iterate for logistic regression under hMCAR missingness across multiple datasets.
        Each imputation-based SGD baseline is compared with its first-order Richardson-SGD counterpart.
        Richardson systematically lowers the final parameter error, with particularly clear gains on the Breast Cancer dataset~\citep{breast_cancer_14}, where stronger imputers lead to substantially smaller Richardson-corrected errors.
    }
    \label{fig:logvarious}
\end{figure}

\paragraph{Additional experiments.}
Appendix~\ref{app:additional-experiments} extends the numerical study beyond logistic regression to several other Generalized Linear Models (linear and Poisson), datasets, imputation rules, and missingness mechanisms.
Across settings, Richardson consistently improves the considered imputations and remains effective beyond the one-pass regime.
The gains are largest in the first epoch, matching the theory of Section~\ref{subsec:one-pass}.
We also show robustness to estimated missingness values by replacing $p$ and $q$ with their estimates in Appendix \ref{app:exp-robust}, and robustness to misspecification of the missingness mechanism by using Richardson-SGD under an assumed hMCAR mechanism while the true mechanism is sMAR in Appendix \ref{app:exp-misspecification}.

\section{Conclusion}
\label{sec:conclusion}

We introduced Richardson-SGD, a simple debiasing method for stochastic gradient learning computed on imputed data. For arbitrary parametric losses and data-independent imputation rules, we establish that the imputation-induced gradient bias admits a first-order expansion in the missingness vector $p$.
We propose the
Richardson-SGD procedure, which turns this structure into an algorithm by deliberately adding controlled missingness. 
This cancels the leading bias term, reducing gradient bias from $O(\|p\|)$ to $O(\|p\|^2)$ and the one-pass SGD error floor from $O(\|p\|^2)$ to $O(\|p\|^4)$.
Our experiments show that one-step Richardson-SGD procedure successfully improves the convergence of SGD for a variety of parametric models and imputation methods.
The procedure is lightweight, model-agnostic, and compatible with standard imputation pipelines. Overall, our results show that controlled additional missingness can be more than a nuisance: used carefully, it becomes a practical tool for reducing bias in stochastic learning from incomplete data.

Our debiasing procedure requires generating more missing data with the same distribution as the original sample, but at an increased scale. Doing so is easy for independent hMCAR data, but becomes challenging in the presence of anticorrelation between mask components. In this setting, we are not able to generate more missing data along all coordinates simultaneously while respecting the form of the original missing data distribution. Future research directions are to extend our procedure to such settings. Note however that, in practice, our procedure may be relatively robust to missingness misspecification (Appendix \ref{app:exp-misspecification}), which leaves some hope to establish positive results in such settings.


\bibliography{references.bib}
\bibliographystyle{abbrvnat}


\newpage
\appendix
\setcounter{tocdepth}{1}
\tableofcontents

\newpage

\section{Additional notation and technical preliminaries}
\label{app:notation}

This appendix collects notation and elementary identities used throughout the proofs.

\paragraph{Notations.}
For a mask $m\in\{0,1\}^d$, $S(m)=\{j:m_j=1\}$. We write $G_S(w;X,Y,\xi):=g(w;\mathcal{I}(X^{\NA,S},M_S,\xi),Y)$ for the gradient when exactly the coordinates in $S$ are declared missing, and $G_\varnothing=g(w;X,Y)$. We use $\Delta_S G_\varnothing:=\prod_{j\in S}(T_j-I)G_\varnothing$, where $T_j$ replaces a sample by its version with coordinate $j$ declared missing. We write $a_j(V)\equiv 1$ under hMCAR and $a_j(V)=q_j(V)$ under sMAR.

\paragraph{Inclusion--exclusion identity.}
For every $S\subseteq[d]$, $\Delta_S G_\varnothing=\sum_{T\subseteq S}(-1)^{|S|-|T|}G_T$, and inversely $G_S=\sum_{T\subseteq S}\Delta_T G_\varnothing$ (Lemma~\ref{lem:det-mask-expansion}). This is the discrete-difference identity that drives the bias expansion.

\section{Proofs for the missingness-bias expansion}
\label{app:bias-decomposition}

This appendix proves the structural expansion of the imputation-induced gradient
bias.

First, we prove a purely algebraic identity: the gradient obtained after hiding
any set of coordinates can be decomposed into a sum of finite-difference effects.
These effects isolate what is due to hiding one coordinate, what is due to hiding
two coordinates jointly, and so on.

Second, we average this identity over the random missingness mask. This turns the
finite-difference effects into a bias expansion whose coefficients are
co-missingness probabilities. The first-order terms correspond to single missing
coordinates; the remainder contains all simultaneous missingness effects.

Throughout this appendix, fix a parameter value $w\in\RR^d$. We suppress the
dependence on $w$ whenever this improves readability. By definition, any
imputation rule leaves a fully observed sample unchanged:
\[
    \mathcal I(X,\mathbf 0,\xi)=X .
\]
Thus, when no coordinate is declared missing, the imputed gradient equals the
complete-data gradient. We assume that all finite differences introduced
below are integrable. This is automatic, for instance, if the gradient is continuous, as assumed in the paper.
\subsection{Proof of Proposition  \ref{prop:grad_bias_fo_struct}}

\paragraph{Gradients indexed by deterministic missingness sets.}
For a deterministic set $S\subseteq[d]$, let $X^{\NA,S}$ be the version of
$X$ in which exactly the coordinates in $S$ are replaced by $\NA$. Let
$M_S\in\{0,1\}^d$ be the deterministic mask associated with $S$:
\[
    (M_S)_j=\mathbf 1\{j\in S\}.
\]
We define
\begin{equation}
    G_S
    :=
    g\big(w;\mathcal I(X^{\NA,S},M_S,
    \xi),Y\big).
    \label{eq:GS-clear}
\end{equation}
Thus, $G_S$ is the gradient we would compute if we deliberately declared
exactly the coordinates in $S$ missing and then applied the imputation rule.
In particular,
\[
    G_\varnothing
    =
    g(w;\mathcal I(X,
    \mathbf 0,
    \xi),Y)
    =
    g(w;X,Y).
\]

\paragraph{Finite missingness differences.}
The objects $G_S$ describe gradients under different missingness patterns. To
separate the effect of one coordinate from the extra effect of several
coordinates being missing together, we use finite differences. For every
$S\subseteq[d]$, define
\begin{equation}
    D_S
    :=
    \sum_{T\subseteq S}(-1)^{|S|-|T|}G_T,
    \qquad D_\varnothing:=G_\varnothing .
    \label{eq:DS-def-clear}
\end{equation}
The first examples are
\[
    D_{\{j\}}=G_{\{j\}}-G_\varnothing,
\]
and
\[
    D_{\{j,k\}}
    =
    G_{\{j,k\}}-G_{\{j\}}-G_{\{k\}}+G_\varnothing .
\]
The interpretation is as follows. The term $D_{\{j\}}$ is the direct effect of
hiding coordinate $j$. The term $D_{\{j,k\}}$ is not the full effect of
hiding $j$ and $k$; it is only the additional interaction left after removing
the two separate single-coordinate effects. Higher-order terms $D_S$ have the
same meaning: they isolate the part of the missingness effect that appears only
when all coordinates in $S$ are hidden together.

The following simple result shows that the full effect of hiding the coordinates in $A$ can be
rebuilt by adding all finite-difference effects supported inside $A$.

\begin{lemma}[Deterministic mask expansion]
    \label{lem:det-mask-expansion}
    For every deterministic set $A\subseteq[d]$,
    \begin{equation}
        G_A
        =
        \sum_{S\subseteq A}D_S .
        \label{eq:det-mask-expansion-clear}
    \end{equation}
\end{lemma}

\begin{proof}

    Starting from the definition of $D_S$,
    \[
        \sum_{S\subseteq A}D_S
        =
        \sum_{S\subseteq A}\sum_{T\subseteq S}(-1)^{|S|-|T|}G_T .
    \]
    We now group the terms by $G_T$. A fixed $G_T$ appears only in those sums
    with $T\subseteq S\subseteq A$, so
    \[
        \sum_{S\subseteq A}D_S
        =
        \sum_{T\subseteq A}G_T
        \sum_{S:T\subseteq S\subseteq A}(-1)^{|S|-|T|} .
    \]
    For fixed $T\subseteq A$, write $S=T\cup R$, where
    $R\subseteq A\setminus T$. Then the inner sum becomes
    \[
        \sum_{R\subseteq A\setminus T}(-1)^{|R|}
        =
        (1-1)^{|A\setminus T|},
    \]
    which results from the binomial expansion of the right-hand side term. This term  equals $1$ if $T=A$, and $0$ otherwise.  Therefore, every term cancels
    except $G_A$, proving \eqref{eq:det-mask-expansion-clear}.
\end{proof}

\paragraph{From a deterministic mask to a random mask.}
We now let $M\in\{0,1\}^d$ be the actual random missingness mask. Recall that
$S(M):=\{j:M_j=1\}$ is the set of missing coordinates for the mask $M$. The corresponding
imputed gradient is $G_{S(M)}$.
For each $S\subseteq[d]$, we recall that the conditional co-missingness probability is
\begin{equation}
    \rho_S
    :=
    \EE\left[\prod_{j\in S}M_j\,\middle|\,X,Y\right]
    =
    \PP\big(M_j=1\text{ for all }j\in S\mid X,Y\big),
    \qquad \rho_\varnothing:=1 .
    \label{eq:rhoS-clear}
\end{equation}
Thus, $\rho_{\{j\}}$ is the conditional probability that coordinate $j$ is
missing, while $\rho_{\{j,k\}}$ is the conditional probability that $j$ and
$k$ are missing simultaneously.

\begin{lemma}[Random mask expansion]
    \label{lem:random-mask-expansion}
    Consider any data-independent imputation parametrized by $\xi$, as defined in \eqref{def:imputation}. Let $M\in\{0,1\}^d$ be any missingness mask. Then the corresponding
    imputed gradient $G_{S(M)}$ satisfies
    \begin{equation}
        \EE\left[G_{S(M)}\mid X,Y,
            \xi\right]
        =
        G_\varnothing
        +
        \sum_{\varnothing\ne S\subseteq[d]}\rho_S D_S .
        \label{eq:random-mask-expansion-clear}
    \end{equation}
\end{lemma}

\begin{proof}
    Apply Lemma~\ref{lem:det-mask-expansion} to the random set $A=S(M)$. For a
    fixed realization of the mask,
    \[
        G_{S(M)}
        =
        \sum_{S\subseteq S(M)}D_S .
    \]
    The condition $S\subseteq S(M)$ is equivalent to saying that every coordinate
    in $S$ is missing, namely $M_j=1$ for all $j\in S$. Therefore
    \[
        \mathbf 1\{S\subseteq S(M)\}
        =
        \prod_{j\in S}M_j,
    \]
    which yields
    \[
        G_{S(M)}
        =
        \sum_{S\subseteq[d]}
        \left(\prod_{j\in S}M_j\right)D_S .
    \]
    Conditional on $(X,Y,
        \xi)$, the finite differences $D_S$ are fixed, and only the mask remains random. Moreover, by assumption, we have
    $\xi\perp M\mid(X,Y)$, which leads to
    \[
        \EE\left[\prod_{j\in S}M_j\,\middle|\,X,Y,
            \xi\right]
        =
        \EE\left[\prod_{j\in S}M_j\,\middle|\,X,Y\right]
        =
        \rho_S .
    \]
    Hence
    \[
        \EE\left[G_{S(M)}\mid X,Y,
            \xi\right]
        =
        \sum_{S\subseteq[d]}\rho_S D_S .
    \]
    The term $S=\varnothing$ equals $D_\varnothing=G_\varnothing$, which gives
    \eqref{eq:random-mask-expansion-clear}.
\end{proof}

\paragraph{Bias expansion.}
We can now prove the first-order structure of the gradient bias. The random
imputed gradient used by the learner is
\[
    \hat g(w)=G_{S(M)}.
\]
Since $G_\varnothing=g(w;X,Y)$, we have
$\EE[G_\varnothing]=\nabla L(w)$. Taking expectations in
Lemma~\ref{lem:random-mask-expansion} therefore yields
\begin{equation}
    \Bbias(w,p)
    :=
    \EE[\hat g(w)]-\nabla L(w)
    =
    \sum_{\varnothing\ne S\subseteq[d]}\EE[\rho_S D_S].
    \label{eq:bias-full-subset-expansion-clear}
\end{equation}
This identity is the key building block. It says that the bias is a sum over all
nonempty missingness sets $S$. Each term has two factors:
\begin{itemize}
    \item $\rho_S$, the probability that all coordinates in $S$ are missing;
    \item $D_S$, the incremental gradient effect created by hiding exactly the
          coordinates in $S$, after lower-order effects have been subtracted.
\end{itemize}
Thus, singletons $S=\{j\}$ produce the first-order bias, while sets with
$|S|\ge2$ produce the co-missingness remainders.

\begin{proof}[Proof of Proposition~\ref{prop:grad_bias_fo_struct}]
    Start from the exact expansion
    \eqref{eq:bias-full-subset-expansion-clear}. We separate the singleton terms
    from the terms involving at least two missing coordinates:
    \begin{equation}
        \Bbias(w,p)
        =
        \sum_{j=1}^d\EE[\rho_{\{j\}}D_{\{j\}}]
        +
        \sum_{|S|\ge2}\EE[\rho_S D_S].
        \label{eq:bias-singletons-plus-remainder-clear}
    \end{equation}
    We now identify the singleton probabilities under the mechanisms considered in
    the paper.

    Under hMCAR,
    \[
        \rho_{\{j\}}
        =
        \PP(M_j=1)
        =
        p_j .
    \]
    Under sMAR, with $V=X^{(\Sobs)}$,
    \[
        \rho_{\{j\}}
        =
        \PP(M_j=1\mid V)
        =
        p_jq_j(V).
    \]
    Both cases can be written as
    \begin{equation}
        \rho_{\{j\}}
        =
        p_j a_j(V),
        \quad \textrm{with }
        a_j(V)=
        \begin{cases}
            1,      & \text{hMCAR}, \\
            q_j(V), & \text{sMAR}.
        \end{cases}
        \label{eq:aj-clear}
    \end{equation}
    Substituting \eqref{eq:aj-clear} into the singleton part of
    \eqref{eq:bias-singletons-plus-remainder-clear} gives
    \[
        \sum_{j=1}^d\EE[\rho_{\{j\}}D_{\{j\}}]
        =
        \sum_{j=1}^d p_j\,\EE[a_j(V)D_{\{j\}}].
    \]
    Since $D_{\{j\}}=G_{\{j\}}-G_\varnothing$, this is exactly
    $\Aop(w)p$, where the $j$-th column of $\Aop(w)$ is
    \[
        \Aop_{\cdot j}(w)
        =
        \EE\left[
            a_j(V)
            \big\{
            G_{\{j\}}(w;X,Y,
            \xi)-G_\varnothing(w;X,Y,
            \xi)
            \big\}
            \right].
    \]
    The remaining terms are precisely the co-missingness remainder:
    \begin{equation}
        \Rrem(w,p)
        :=
        \sum_{|S|\ge2}\EE[\rho_S D_S].
        \label{eq:remainder-exact-clear}
    \end{equation}
    Combining the singleton part and the remainder proves
    \[
        \Bbias(w,p)=\Aop(w)p+\Rrem(w,p).
    \]
\end{proof}

\begin{remark}[What the remainder contains]
    \label{rem:dependent-remainder-size}
    The remainder $\Rrem(w,p)$ is the sum  of all interaction terms caused by simultaneous missingness. For example,
    the pair $\{j,k\}$ contributes
    \[
        \EE[\rho_{\{j,k\}}D_{\{j,k\}}].
    \]
    If $M_j$ and $M_k$ are independent and each is missing with probability of
    order $p$, then $\rho_{\{j,k\}}$ is of order $p^2$. If instead the two
    coordinates are always missing together, then $\rho_{\{j,k\}}$ can be of order
    $p$. Thus, without a weak-dependence condition on co-missingness probabilities,
    $\Rrem(w,p)$ may be a first-order term, proportional to $p$.
\end{remark}

\subsection{Proof of Corollary \ref{cor:bias_asymp_indep}}
\begin{proof}
    Assume that the missingness indicators are conditionally independent given the
    variables driving the missingness mechanism. In hMCAR, this is ordinary
    independence. In sMAR, this is conditional independence given $V$.

    Then, for every $S\subseteq[d]$, the probability that all coordinates in
    $S$ are missing factorizes:
    \[
        \rho_S
        =
        \prod_{j\in S}\rho_{\{j\}}
        =
        \prod_{j\in S}p_j a_j(V).
    \]

    Defining
    \begin{equation}
        \mu_S(w)
        :=
        \EE\left[
            \left(\prod_{j\in S}a_j(V)\right)D_S(w)
            \right],
        \label{eq:muS-clear}
    \end{equation}

    we then have the \textbf{multilinear form under independent masks:}
    \begin{equation}
        \Bbias(w,p)
        =
        \sum_{\varnothing\ne S\subseteq[d]}
        \left(\prod_{j\in S}p_j\right)
        \mu_S(w),
        \label{eq:multilinear-bias-clear}
    \end{equation}
    and thus
    \eqref{eq:remainder-exact-clear},
    \begin{align*}
        \Rrem(w,p)
         & =
        \sum_{|S|\ge2}
        \left(\prod_{j\in S}p_j\right)
        \mu_S(w).
    \end{align*}
    Every term in this sum contains at least two factors $p_j$. Since the
    dimension is fixed and the finite differences are integrable, there exists a
    finite constant $C_w$, depending on $w$ but not on $p$, such that
    \[
        \|\Rrem(w,p)\|
        \le
        C_w
        \sum_{|S|\ge2}\prod_{j\in S}p_j .
    \]
    The last sum is $O(\|p\|^2)$ as $\|p\|\to0$, because each product contains at
    least two entries of $p$. Therefore
    \[
        \|\Rrem(w,p)\|=O(\|p\|^2),
        \qquad
        \big\|\Bbias(w,p)-\Aop(w)p\big\|=O(\|p\|^2).
    \]
\end{proof}

\begin{remark}
    The coefficient $\mu_S(w)$ is the average $|S|$-way missingness interaction:
    it is the effect of declaring all coordinates in $S$ missing, after all lower
    order effects have been removed by inclusion-exclusion. This formula is the
    reason Richardson extrapolation applies: the bias is organized by powers of the
    missingness scale.
\end{remark}

\section{Proofs for Richardson correction}
\label{app:richardson-proofs}

This appendix collects the proofs of the Richardson-extrapolation results: the joint law of the further-thinned mask, first- and higher-order bias cancellation, the subset-based variant, the plug-in mechanism, and the linear-regression case study. We close with explicit GLM bias formulas.

\subsection{Joint law of the further-thinned mask}
\label{app:thinning-law}

Fix $C>1$ and assume $C\,p_ja_j(V)\le 1$ a.s.\ for every $j\in\Smiss$, with $a_j$ as in~\eqref{eq:aj-clear}. For $j\in\Smiss$, draw $r_j$ as in~\eqref{eq:thinning}, conditionally independent across $j$ given $(X,M^{(p)})$. Define $M_j^{(Cp)}:=1-(1-M_j^{(p)})\,r_j$.

Since the imputation rule is conditionally independent of $M^{(p)}$ given $(X,Y)$, by construction we have $\PP(M_j^{(p)}=0\mid X,Y)=1-p_ja_j(V)$, where $V:=X^{(\Sobs)}$. Recall that $M_j^{(Cp)}=0$ iff $M_j^{(p)}=0$ AND $r_j=1$. Hence,
\begin{align*}
    \PP\big(M_j^{(Cp)}=0\,\big|\,V,Y\big)
    \; & =\;\PP\big(M_j^{(p)}=0\,\big|\,V,Y\big)\;\PP\big(r_j=1\,\big|\,V,Y,M_j^{(p)}=0\big) \\
    \; & =\;\big(1-p_ja_j(V)\big)\;\frac{1-Cp_ja_j(V)}{1-p_ja_j(V)}                          \\
    \; & =\;1-Cp_ja_j(V).
\end{align*}
Due to the conditional independence of $M^{(p)}$ given $V$, we obtain the conditional independence of $\{M_j^{(Cp)}\}_{j\in\Smiss}$ given $V$. Hence, $M^{(Cp)}$ has the same conditional law as an independent mask drawn at scale $Cp$.

\subsection{Proof of Proposition~\ref{prop:first-order-debias}}
\label{app:proof-cor-first-order}

\begin{proof}Apply Proposition~\ref{prop:grad_bias_fo_struct} at scales $p$ and $Cp$. Both biases admit the decomposition $\Bbias(w,\cdot)=\Aop(w)\,\cdot+\Rrem(w,\cdot)$, with the same operator $\Aop(w)$ (since by~\eqref{eq:bias-Aj}, $\Aop$ does not depend on $p$). Substituting into~\eqref{eq:richardson-grad},
    \begin{align*}
        \EE\left[\hat g_C^{\mathrm{R}}(w)\right]-\nabla L(w)
        \; & =\;
        \frac{C\,\Bbias(w,p)-\Bbias(w,Cp)}{C-1} \\
        \; & =\;
        \frac{C\,\Aop(w)\,p-\Aop(w)\,(Cp)}{C-1}
        \;+\;
        \frac{C\,\Rrem(w,p)-\Rrem(w,Cp)}{C-1}.
    \end{align*}
    Linearity of $\Aop$ yields $C\,\Aop(w)\,p-\Aop(w)\,(Cp)=0$, so only the remainder survives:
    \[
        \EE\left[\hat g_C^{\mathrm{R}}(w)\right]-\nabla L(w)
        \;=\;
        \frac{C\,\Rrem(w,p)-\Rrem(w,Cp)}{C-1}.
    \]
    Under conditional independence of the $\{M_j\}_{j\in\Smiss}$ given $V$, Proposition~\ref{prop:grad_bias_fo_struct} gives $\|\Rrem(w,p)\|=O(\|p\|^2)$ and, by the same bound applied at scale $Cp$, $\|\Rrem(w,Cp)\|=O(C^2\|p\|^2)=O(\|p\|^2)$. Combining,

    \begin{align*}
        \|\,\EE[\hat g_C^{\mathrm{R}}(w)]-\nabla L(w)\,\big\|\;=\;O(\|p\|^2),
    \end{align*}

    which is \eqref{eq:first-order-cancel}.
\end{proof}

\subsection{Proof of Corollary~\ref{cor:k-order} (higher-order cancellation)}
\label{app:proof-cor-k-order}

Under independent masks, $\Bbias(w,p)=\sum_{\varnothing\ne S\subseteq[d]}\big(\prod_{j\in S}p_j\big)\mu_S(w)$ from~\eqref{eq:bias-poly}. Group terms by $|S|$:
\[
    \Bbias(w,p)\;=\;\sum_{m=1}^d\beta_m(w,p),
    \qquad\text{with}\qquad
    \beta_m(w,p)\;:=\;\sum_{|S|=m}\Big(\prod_{j\in S}p_j\Big)\mu_S(w),
\]
so that $\beta_m(w,\cdot)$ is homogeneous of degree $m$, i.e.\ $\beta_m(w,Cp)=C^m\beta_m(w,p)$. For a sequence of expansion factors $1=C_0<C_1<\cdots<C_k$ with $C_kp_ja_j(V)\le 1$ a.s.,
\[
    \Bbias(w,C_\ell p)\;=\;\sum_{m=1}^d C_\ell^{m}\,\beta_m(w,p),\qquad \ell=0,\dots,k.
\]
The Vandermonde system
\[
    \sum_{\ell=0}^k\alpha_\ell\;=\;1,
    \qquad
    \sum_{\ell=0}^k\alpha_\ell\,C_\ell^{m}\;=\;0,\quad m=1,\dots,k,
\]
admits a unique solution $\alpha\in\RR^{k+1}$ since the matrix $(C_\ell^m)_{\ell,m=0}^k$ is a non-singular Vandermonde. With this choice of $\alpha$,
\begin{align*}
    \EE\left[\hat g^{[k]}(w)\right]-\nabla L(w)
    \; & =\;\sum_{\ell=0}^k\alpha_\ell\,\Bbias(w,C_\ell p)                           \\
    \; & =\;\sum_{m=1}^d\Big(\sum_{\ell=0}^k\alpha_\ell C_\ell^m\Big)\beta_m(w,p)    \\
    \; & =\;\sum_{m=k+1}^d\Big(\sum_{\ell=0}^k\alpha_\ell C_\ell^m\Big)\beta_m(w,p),
\end{align*}
where the last equality uses $\sum_\ell\alpha_\ell C_\ell^m=0$ for $m=1,\dots,k$ and $\sum_\ell\alpha_\ell C_\ell^0=1$ but the $m=0$ term contributes $\sum_\ell\alpha_\ell\,\Bbias(w,0)=0$ since $\Bbias(w,0)=0$. Each $\beta_m(w,p)$ is bounded by $\|p\|_\infty^m\,\sum_{|S|=m}\|\mu_S(w)\|=O(\|p\|^m)$, hence
\[
    \big\|\EE[\hat g^{[k]}(w)]-\nabla L(w)\big\|\;=\;O(\|p\|^{k+1}).
\]
Finally, let $d_{\mathrm{miss}}=\#\{j:p_j>0\}$. When $k=d_{\mathrm{miss}}$, every $S$ with $|S|>d_{\mathrm{miss}}$ has at least one coordinate with $p_j=0$, so $\prod_{j\in S}p_j=0$ and $\beta_m(w,p)=0$ for $m>d_{\mathrm{miss}}$. The residual bias vanishes identically: $\EE[\hat g^{[d_{\mathrm{miss}}]}(w)]=\nabla L(w)$. The argument under sMAR is identical, with $p_j$ replaced by $p_ja_j(V)$ inside the expectation defining $\mu_S(w)$.\hfill$\square$

\subsection{Proof of Proposition~\ref{prop:plug-in} (plug-in mechanism)}
\label{app:plug-in}

Let $\lambda_j(V):=p_jq_j(V)$ and $\hat\lambda_j(V):=\hat p_j\hat q_j(V)$. We write $M^{\tilde R}$ for the further-thinned mask produced by the plug-in rule with intensities $\hat\lambda_j$, where $\tilde r_j\sim\mathrm{Bernoulli}((1-C\hat\lambda_j(V))/(1-\hat\lambda_j(V)))$.

\paragraph{Step 1: Effective intensity of the plug-in further-thinned mask.}

Conditioning on $V$ and using the conditional independence of $\tilde r_j$ and $M^{(p)}$ given $V$, and the fact that $M_j^{\tilde R}=0$ iff $M_j^{(p)}=0$ AND $\tilde r_j=1$,
\begin{align*}
    \PP\big(M_j^{\tilde R}=1\,\big|\,V\big)
     & =1-\PP\big(M_j^{(p)}=0,\,\tilde r_j=1\,\big|\,V\big) \\
     & =1-\big(1-\lambda_j(V)\big)\,
    \frac{1-C\hat\lambda_j(V)}{1-\hat\lambda_j(V)}.
\end{align*}

Define the \emph{effective intensity} after plug-in thinning by
\[
    \tilde\lambda_j(V):=\PP(M_j^{\tilde R}=1\mid V).
\]
Expanding the previous display over the common denominator
$1-\hat\lambda_j(V)$, we obtain
\begin{align}
    \tilde\lambda_j(V)
     & =
    \frac{
        1-\hat\lambda_j(V)
        -\big(1-\lambda_j(V)\big)
        \big(1-C\hat\lambda_j(V)\big)
    }{
        1-\hat\lambda_j(V)
    } \notag                 \\
     & =
    \frac{
        \lambda_j(V)
        +(C-1)\hat\lambda_j(V)
        -C\lambda_j(V)\hat\lambda_j(V)
    }{
        1-\hat\lambda_j(V)
    } \notag                \\
     & =
    \frac{
        C\lambda_j(V)\big(1-\hat\lambda_j(V)\big)
        +(C-1)\big(\hat\lambda_j(V)-\lambda_j(V)\big)
    }{
        1-\hat\lambda_j(V)
    } \notag \\
     & =
    C\,\lambda_j(V)
    +(C-1)\,
    \frac{\hat\lambda_j(V)-\lambda_j(V)}
    {1-\hat\lambda_j(V)}.
    \label{eq:plug-in-eff-intensity}
\end{align}
The first term is the desired intensity of a new draw at scale $Cp$; the second is the plug-in error. Setting
\[
    e_j^{(C)}(V)\;:=\;\tilde\lambda_j(V)-C\,\lambda_j(V)\;=\;(C-1)\,\frac{\hat\lambda_j(V)-\lambda_j(V)}{1-\hat\lambda_j(V)},
\]
we have $\tilde\lambda(V)=C\lambda(V)+e^{(C)}(V)$.

\paragraph{Step 2: Bias of the plug-in Richardson gradient.}
By the same expansion as in Proposition~\ref{prop:grad_bias_fo_struct}, the
singleton part of the bias is obtained by multiplying the singleton gradient gap
by the corresponding conditional missingness probability. For the original mask,
this probability is
\[
    \PP(M_j=1\mid V)=p_j a_j(V),
\]
whereas for the plug-in further-thinned mask, Step~1 gives
\[
    \PP(M_j^{\tilde R}=1\mid V)
    =
    \tilde\lambda_j(V)
    =
    C p_j a_j(V)+e_j^{(C)}(V).
\]
Applying the inclusion--exclusion expansion (eq:bias-full-subset-expansion-clear) separately to each of the two stochastic gradients in $\hat g_{C,\hat\lambda}^{\mathrm{R}}(w)=(C\,\hat g^{(p)}(w)-\hat g^{(Cp,\hat\lambda)}(w))/(C-1)$, and using $\rho_{\{j\}}=p_ja_j(V)$ for the original mask and $\PP(M_j^{\tilde R}=1\mid V)=Cp_ja_j(V)+e_j^{(C)}(V)$ for the plug-in further-thinned mask, the singleton contributions to the two biases are
\[
    \sum_{j=1}^d p_j\,\EE\!\left[a_j(V)D_{\{j\}}\right]
    \quad\text{and}\quad
    \sum_{j=1}^d \big(Cp_ja_j(V)+e_j^{(C)}(V)\big)\,\EE\!\left[D_{\{j\}}\mid V\right],
\]
respectively. In the Richardson combination, the deterministic $Cp_ja_j(V)$ contributions cancel exactly, leaving
\begin{align*}
    \EE[\hat g_{C,\hat\lambda}^{\mathrm{R}}(w)]-\nabla L(w)
     & =
    -\frac{1}{C-1}
    \sum_{j=1}^d
    \EE\!\left[
             e_j^{(C)}(V)
             \big\{
             G_{\{j\}}(w;X,Y,\xi)-G_\varnothing(w;X,Y,\xi)
             \big\}
             \right] \\
     & \qquad
    +
    \frac{
        C\,\Rrem(w,p)-\Rrem^{\tilde R}(w)
    }{C-1},
\end{align*}
where $\Rrem^{\tilde R}(w):=\sum_{|S|\ge 2}\EE[\rho_S^{\tilde R}D_S]$ is the co-missingness remainder evaluated at the plug-in further-thinned mask. Thus the only remaining first-order contribution is the plug-in intensity error $e_j^{(C)}(V)$. Under the assumed $L^2$ bound on the singleton gradient gaps, there exists $G_\star$ such that
\[
    \big\|
    G_{\{j\}}(w;X,Y,\xi)-G_\varnothing(w;X,Y,\xi)
    \big\|_{L^2}
    \le
    G_\star,
    \qquad \forall j\in[d].
\]
Therefore,

\[
    \left\|
    \sum_{j=1}^d
    \EE\!\left[
        e_j^{(C)}(V)
        \big\{
        G_{\{j\}}(w;X,Y,\xi)-G_\varnothing(w;X,Y,\xi)
        \big\}
        \right]
    \right\|
    =
    O\!\left(\|e^{(C)}\|_\infty\right).
\]
Moreover, the co-missingness remainder is
\[
    O\!\left(
    \|p\|^2
    +
    \|p\|\,\|e^{(C)}\|_\infty
    +
    \|e^{(C)}\|_\infty^2
    \right).
\]
Combining gives ,
\begin{equation}
    \big\|
    \EE[\hat g_{C,\hat\lambda}^{\mathrm{R}}(w)]-\nabla L(w)
    \big\|
    =
    O\!\left(
    \|e^{(C)}\|_\infty
    +
    \|p\|^2
    +
    \|p\|\,\|e^{(C)}\|_\infty
    +
    \|e^{(C)}\|_\infty^2
    \right). \label{eq:plug-in-intermediate}
\end{equation}

\paragraph{Step 3: Bound on the plug-in error.}
Using
\begin{align*}
    |\hat p_j\hat q_j(V)-p_jq_j(V)| & =\left|(\hat p_j-p_j)q_j(V)+\hat p_j(\hat q_j(V)-q_j(V))\right| \\
                                    & \le |\hat p_j-p_j|\,|q_j(V)|+|\hat p_j|\,|\hat q_j(V)-q_j(V)|
\end{align*}
and the bounds $|q_j(V)|\le 1$, $|\hat p_j|\le \|p\|_\infty+\delta_p$,
\[
    |\hat\lambda_j(V)-\lambda_j(V)|
    \;\le\;\delta_p+(\|p\|_\infty+\delta_p)\,\delta_q.
\]
Since $\hat\lambda_j(V)\le\rho<1$,
\[
    |e_j^{(C)}(V)|
    \;\le\;\frac{C-1}{1-\rho}\left[\delta_p+(\|p\|_\infty+\delta_p)\,\delta_q\right],
\]
hence, defining $\|e^{(C)}\|_\infty:=\sup_{j\in\Smiss,\,v}|e_j^{(C)}(v)|$,
\[
    \|e^{(C)}\|_\infty\;=\;O\big(\delta_p+\|p\|_\infty\delta_q+\delta_p\delta_q\big).
\]

\paragraph{Step 4: Concluding.}
Substituting the bound on $\|e^{(C)}\|_\infty$ into~\eqref{eq:plug-in-intermediate} and simplifying,
\[
    \big\|\EE[\hat g_{C,\hat\lambda}^{\mathrm{R}}(w)]-\nabla L(w)\big\|
    \;=\;O\Big(\|p\|^2+\delta_p+\|p\|_\infty\delta_q+\delta_p\delta_q\Big),
\]
where the implicit constants depend only on $C$, $\rho$, and $G_\star$. In the MCAR case $q_j\equiv 1$, $\delta_q=0$ and the bound collapses to $O(\|p\|^2+\delta_p)$.\hfill$\square$

\section{Bias formulas for specific generalized linear models}
\label{app:glm_bias_formulas}

We record explicit expressions for the leading-order population bias $\Aop(w)\,p$ in three GLMs under heterogeneous MCAR with zero imputation. When the missingness factors are independent, $\Aop(w)\,p$ is the only leading term in $p$. In full generality, however, some terms in the remainder may also be linear. In all cases, Richardson-SGD eliminates all linear terms, whether or not additional linear contributions appear in the remainder.

Notably, the bias of linear regression is a polynomial of total degree at most $2$ in $p$, whereas logistic and Poisson regression generally exhibit full-degree bias, up to degree $d$.

\paragraph{Linear regression (squared loss).}
For $\ell(w;x,y)=\tfrac12(w^\top x-y)^2$ and zero imputation,
\[
    \big(\Aop(w)\,p\big)_j\;=\;-p_j\,\nabla_jL(w)\;-\;\sum_{k\ne j}p_k\,S_{jk}w_k,\qquad S:=\EE[XX^\top].
\]
The detailed derivation, including the exact non-asymptotic version, is reproduced in Appendix~\ref{app:linreg}.

\paragraph{Logistic regression.}
For $\ell(w;x,y)=\log(1+e^{-y\,w^\top x})$ with $y\in\{-1,+1\}$, the gradient is $g(w;x,y)=-y\,\sigma(-y\,w^\top x)\,x$ where $\sigma$ is the logistic function. Under zero imputation and heterogeneous MCAR,
\[
    \big(\Aop(w)\,p\big)_j\;=\;-p_j \nabla_j L(w)
    +\sum_{k\ne j}p_k\,\EE\left[Y\big(\sigma(-Y\,w^\top X)-\sigma(-Y\,w^\top X^{(-k)})\big)X_j\right],
\]
where $X^{(-k)}$ is $X$ with $X_k$ replaced by $0$.

\paragraph{Poisson regression.}
For $\ell(w;x,y)=e^{w^\top x}-y\,w^\top x$, the gradient is $g(w;x,y)=(e^{w^\top x}-y)\,x$, hence
\[
    \big(\Aop(w)\,p\big)_j\;=\;-p_j\,\nabla_jL(w)\;+\;\sum_{k\ne j}p_k\,\EE\left[(e^{w^\top X^{(-k)}}-e^{w^\top X})X_j\right].
\]
All three expressions are obtained by substituting the corresponding loss into~\eqref{eq:bias-Aj}. They share the same structural form: a coordinate-wise diagonal contribution $-p_j\nabla_jL(w)$, plus an off-diagonal correction.

\section{Linear regression: a transparent case study}
\label{app:linreg}

The goal of this appendix is to show on the simplest GLM that, under heterogeneous MCAR with independent masks and zero imputation, the population gradient bias is a polynomial of degree at most $2$ in $p$. By Corollary~\ref{cor:k-order}, second-order Richardson with two factors $C_1<C_2$ therefore cancels this bias \emph{exactly}, while a single Richardson step already reduces it from $O(\|p\|)$ to $O(\|p\|^2)$.

\paragraph{Setting.}
We work at the single-observation level with squared loss,
\[
    \ell(w;x,y)=\tfrac12(w^\top x-y)^2,\qquad g(w;x,y)=(w^\top x-y)\,x.
\]
The population risk is $L(w)=\tfrac12\EE[(w^\top X-Y)^2]$ with $\nabla L(w)=Sw-b$, $S:=\EE[XX^\top]$, $b:=\EE[YX]$. We assume heterogeneous MCAR with independent mask coordinates, $\PP(M_j=1)=p_j$, and zero imputation $\tilde X=(1-M)\odot X$. The imputed gradient is $\hat g(w):=g(w;\tilde X,Y)=(w^\top\tilde X-Y)\tilde X$.

\subsection{Sample-conditional and population biases}

\begin{proposition}[Sample-conditional bias]
    \label{prop:sample-bias-app}
    Under heterogeneous MCAR with independent masks, for each $j\in[d]$,
    \begin{equation}
        \EE_M\left[\hat g_j(w;\tilde X,Y)\,\big|\,X,Y\right]-g_j(w;X,Y)
        \;=\;-p_jX_j^2w_j-\sum_{k\ne j}(p_j+p_k-p_jp_k)X_jX_kw_k+p_jYX_j,
        \label{eq:sample-bias-form1-app}
    \end{equation}
    or equivalently
    \begin{equation}
        \EE_M\left[\hat g_j(w;\tilde X,Y)\,\big|\,X,Y\right]-g_j(w;X,Y)
        \;=\;-p_j\,g_j(w;X,Y)-(1-p_j)\sum_{k\ne j}p_k\,X_jX_kw_k.
        \label{eq:sample-bias-form2-app}
    \end{equation}
\end{proposition}

\begin{proof}
    With $\omega_j:=1-M_j$ and $\tilde X_j=\omega_jX_j$, the imputed $j$-th gradient is
    \[
        \hat g_j(w;\tilde X,Y)=(w^\top\tilde X-Y)\tilde X_j=\Big(\sum_kw_k\omega_kX_k-Y\Big)\omega_jX_j.
    \]
    Conditioning on $(X,Y)$ and using independence of the mask coordinates ($\EE[\omega_j]=\EE[\omega_j^2]=1-p_j$ and $\EE[\omega_j\omega_k]=(1-p_j)(1-p_k)$ for $k\ne j$),
    \[
        \EE_M\left[\hat g_j(w;\tilde X,Y)\,\big|\,X,Y\right]
        \;=\;(1-p_j)X_j^2w_j+\sum_{k\ne j}(1-p_j)(1-p_k)X_jX_kw_k-(1-p_j)YX_j.
    \]
    Subtracting the complete-data gradient $g_j(w;X,Y)=X_j^2w_j+\sum_{k\ne j}X_jX_kw_k-YX_j$ yields~\eqref{eq:sample-bias-form1-app} after expanding $(1-p_j)(1-p_k)-1=-(p_j+p_k-p_jp_k)$. To obtain~\eqref{eq:sample-bias-form2-app}, factor $-p_j$ in front of $g_j(w;X,Y)$:
    \begin{align*}
         & -p_jX_j^2w_j-p_j\sum_{k\ne j}X_jX_kw_k+p_jYX_j-(1-p_j)\sum_{k\ne j}p_kX_jX_kw_k \\
         & \qquad=-p_j\,g_j(w;X,Y)-(1-p_j)\sum_{k\ne j}p_kX_jX_kw_k.\qedhere
    \end{align*}
\end{proof}

\begin{corollary}[Population bias of zero-imputed linear regression]
    \label{cor:population-bias-app}
    Under heterogeneous MCAR with independent masks, with $B_j(w;p):=\EE\left[\hat g_j(w;\tilde X,Y)\right]-\nabla_jL(w)$,
    \begin{equation}
        B_j(w;p)\;=\;-p_j\,\nabla_jL(w)\;-\;(1-p_j)\sum_{k\ne j}p_k\,S_{jk}w_k,
        \label{eq:population-bias-app}
    \end{equation}
    hence $\|B(w;p)\|=O(\|p\|)$.
\end{corollary}

\begin{proof}
    Take expectation in~\eqref{eq:sample-bias-form2-app} and use $\EE[g_j(w;X,Y)]=\nabla_jL(w)$ and $\EE[X_jX_k]=S_{jk}$ to obtain~\eqref{eq:population-bias-app}. The norm bound follows from $|B_j(w;p)|\le p_j|\nabla_jL(w)|+\sum_{k\ne j}p_k|S_{jk}w_k|\le \|p\|_\infty\big(|\nabla_jL(w)|+\sum_k|S_{jk}w_k|\big)$.
\end{proof}

\subsection{Polynomial structure and exact debiasing in two Richardson steps}

We now make explicit that the population gradient bias is a polynomial of degree at most $2$ in $p$, hence is annihilated exactly by second-order Richardson with two factors.

\paragraph{The bias is degree-$2$ in $p$.}
Under heterogeneous MCAR with independent masks, from~\eqref{eq:population-bias-app},
\begin{equation}
    B_j(w;p)
    \;=\;\underbrace{\Big(-p_j\,\nabla_jL(w)-\sum_{k\ne j}p_k\,S_{jk}w_k\Big)}_{=:L_j(w;p),\ \text{linear in }p}
    \;+\;\underbrace{p_j\sum_{k\ne j}p_k\,S_{jk}w_k}_{=:Q_j(w;p),\ \text{quadratic in }p}.
    \label{eq:linreg-bias-decomp}
\end{equation}
The same conclusion follows from the general expansion~\eqref{eq:bias-poly} of Section~\ref{sec:bias-structure}, since for the squared loss the iterated finite differences $\Delta_S G_\varnothing$ vanish identically for $|S|\ge 3$ (the gradient is bilinear in $X$).

\paragraph{One Richardson step removes the linear part.}
\begin{proposition}[First-order Richardson cancellation, heterogeneous squared-loss MCAR]
    \label{prop:heterogeneous-richardson-app}
    Under heterogeneous MCAR with independent masks, for $C>1$ with $Cp_j<1$ for every $j$, the first-order Richardson gradient~\eqref{eq:richardson-grad} satisfies
    \[
        \EE\left[\hat g_j^{\mathrm{R}}(w)\right]-\nabla_jL(w)
        \;=\;-C\,p_j\sum_{k\ne j}p_k\,S_{jk}w_k\;=\;-C\,Q_j(w;p).
    \]
    In particular, $\big\|\EE[\hat g^{\mathrm{R}}(w)]-\nabla L(w)\big\|=O(\|p\|^2)$, while the uncorrected bias is $O(\|p\|)$. If $p_j=0$ for some $j$, then $\EE[\hat g_j^{\mathrm{R}}(w)]=\nabla_jL(w)$, i.e.\ the Richardson bias vanishes in coordinate $j$.
\end{proposition}

\begin{proof}
    Apply $\EE[\hat g_j^{\mathrm{R}}]-\nabla_jL=(C\,B_j(w;p)-B_j(w;Cp))/(C-1)$ to~\eqref{eq:linreg-bias-decomp}. Using \begin{align*}L_j(w;Cp)=C\,L_j(w;p) \quad \textrm{and} \quad  Q_j(w;Cp)=C^2\,Q_j(w;p),
    \end{align*}
    we have
    \begin{align*}
        \frac{C\,B_j(w;p)-B_j(w;Cp)}{C-1}
         & \;=\;\frac{C\,L_j(w;p)-C\,L_j(w;p)}{C-1}+\frac{C\,Q_j(w;p)-C^2Q_j(w;p)}{C-1} \\
         &
        \;=\;-C\,Q_j(w;p),
    \end{align*}
    which gives the claim.
\end{proof}

\paragraph{Two Richardson steps cancel the bias exactly.}
\begin{corollary}[Exact Richardson debiasing for linear regression]
    \label{cor:exact-linreg-app}
    Suppose heterogeneous MCAR with independent masks. Let $1=C_0<C_1<C_2$ be three expansion factors with $C_2p_j<1$ for every $j$. The unique transposed Vandermonde solution
    \[
        (\alpha_0,\alpha_1,\alpha_2)\quad\text{of}\quad
        \begin{cases}
            \alpha_0+\alpha_1+\alpha_2=1,       \\
            \alpha_0+\alpha_1C_1+\alpha_2C_2=0, \\
            \alpha_0+\alpha_1C_1^2+\alpha_2C_2^2=0,
        \end{cases}
    \]
    gives a second-order Richardson gradient $\hat g^{[2]}=\sum_\ell\alpha_\ell\hat g^{(C_\ell p)}$ with $\EE[\hat g^{[2]}(w)]=\nabla L(w)$ exactly. The same statement holds in independent sMAR (conditional on $V$), after replacing $p_j$ by $p_ja_j(V)$ inside the expectations defining $L_j$ and $Q_j$.
\end{corollary}

\begin{proof}
    The bias~\eqref{eq:linreg-bias-decomp} is a polynomial of degree $\le 2$ in $p$ with no constant term, so $\Bbias(w,Cp)=C\,L(w;p)+C^2\,Q(w;p)$ for any $C>0$. Applying $\sum_\ell\alpha_\ell\Bbias(w,C_\ell p)$ and using the Vandermonde conditions, both $L$ and $Q$ contributions vanish.
\end{proof}

This recovers, in our framework, the closed-form debiasing of \citet{sportisse2020debiasing} for linear regression with squared loss under independent MCAR, and extends it to the sMAR mechanisms of Section~\ref{sec:setting}, where no closed-form bias is available.

\section{One-pass biased-SGD convergence consequences}
\label{app:onepass}
We provide a simple proof for the non-averaged iterates of SGD with biased gradients. The rates stated in the main text are then recovered by applying this result with the bias corresponding to each method. Note that one could also aim for similar guarantees for averaged SGD with a broader class of step sizes, namely $\eta_t \propto t^{-a}$ with $a \in (1/2,1)$, which notably does not require prior knowledge of the strong convexity constant of the loss \citep{polyak1992acceleration,bach2014adaptivity}. We do not pursue this direction here, as our main focus is bias reduction, and the theorem below already illustrates its practical benefit.

\begin{proposition}[One-pass SGD with Bias]
    \label{prop:general-one-pass}
    Assume $L$ is $\alpha$-strongly convex and $\beta$-smooth, and let
    $w^\star$ be its unique minimizer. Run one-pass SGD over $n$ i.i.d. samples,
    \[
        w_{i+1}=w_i-\eta_i\hat g_i(w_i),
    \]
    with step sizes
    \[
        \eta_i=\frac{c}{i+\gamma},
        \qquad
        \alpha c>1,
    \]
    where $\gamma$ is large enough that
    \[
        \eta_i\le \frac{\alpha}{6\beta^2}
        \qquad \text{for all } i .
    \]
    Suppose that there exists $B(p)$ such that the imputed gradient satisfies, uniformly along the trajectory,
    \[
        \big\|\mathbb E[\hat g_i(w_i)\mid w_i]-\nabla L(w_i)\big\|
        \le B(p),
    \]
    and
    \[
        \mathbb E\!\left[
            \big\|
            \hat g_i(w_i)-\mathbb E[\hat g_i(w_i)\mid w_i]
            \big\|^2
            \,\middle|\, w_i
            \right]
        \le \sigma^2
        \qquad \text{a.s.}
    \]
    Then
    \[
        \mathbb E\|w_n-w^\star\|^2
        =
        O\!\left(B(p)^2\right)+O\!\left(\frac{\sigma^2}{n}\right).
    \]
    In particular, if $\sigma^2=O(1)$, then
    \begin{equation}\label{eq:generic-onepass}
        \mathbb E\|w_n-w^\star\|^2
        =
        O\!\left(B(p)^2\right)+O(1/n).
    \end{equation}
\end{proposition}

\begin{proof}
    Write
    \[
        \mathbb E[\hat g_i(w_i)\mid w_i]
        =
        \nabla L(w_i)+\Delta_i,
        \qquad
        \|\Delta_i\|\le B(p).
    \]
    Also write
    \[
        \hat g_i(w_i)
        =
        \nabla L(w_i)+\Delta_i+\xi_i,
        \qquad
        \mathbb E[\xi_i\mid w_i]=0,
        \qquad
        \mathbb E[\|\xi_i\|^2\mid w_i]\le \sigma^2 .
    \]
    Let
    \[
        \delta_i := \mathbb E\|w_i-w^\star\|^2 .
    \]
    Since $w^\star$ minimizes $L$, $\nabla L(w^\star)=0$. Expanding one SGD step
    gives
    \[
        \begin{aligned}
            \mathbb E[\|w_{i+1}-w^\star\|^2\mid w_i]
             & =
            \|w_i-w^\star\|^2
            -2\eta_i\langle w_i-w^\star,\nabla L(w_i)\rangle \\
             & \quad
            -2\eta_i\langle w_i-w^\star,\Delta_i\rangle
            +\eta_i^2
            \mathbb E[
                            \|\nabla L(w_i)+\Delta_i+\xi_i\|^2
                            \mid w_i
                        ] .
        \end{aligned}
    \]
    By strong convexity,
    \[
        \langle w_i-w^\star,\nabla L(w_i)\rangle
        \ge
        \alpha\|w_i-w^\star\|^2 .
    \]
    By Young's inequality,
    \[
        2|\langle w_i-w^\star,\Delta_i\rangle|
        \le
        \frac{\alpha}{2}\|w_i-w^\star\|^2
        +
        \frac{2}{\alpha}B(p)^2 .
    \]
    By smoothness and $\nabla L(w^\star)=0$,
    \[
        \|\nabla L(w_i)\|
        \le
        \beta\|w_i-w^\star\|.
    \]
    Moreover, using $\mathbb E[\xi_i\mid w_i]=0$,
    \[
        \begin{aligned}
            \mathbb E[
                            \|\nabla L(w_i)+\Delta_i+\xi_i\|^2
                            \mid w_i
                        ]
             & =
            \|\nabla L(w_i)+\Delta_i\|^2
            +
            \mathbb E[\|\xi_i\|^2\mid w_i] \\
             & \le
            2\beta^2\|w_i-w^\star\|^2
            +
            2B(p)^2
            +
            \sigma^2 .
        \end{aligned}
    \]
    Combining these bounds yields
    \[
        \begin{aligned}
            \mathbb E[\|w_{i+1}-w^\star\|^2\mid w_i]
             & \le
            \left(
            1-\frac{3\alpha\eta_i}{2}
            +2\beta^2\eta_i^2
            \right)
            \|w_i-w^\star\|^2 \\
             & \quad
            +
            \frac{2\eta_i}{\alpha}B(p)^2
            +
            \eta_i^2\{2B(p)^2+\sigma^2\}.
        \end{aligned}
    \]
    Since $\eta_i\le \alpha/(6\beta^2)$, we have
    \[
        2\beta^2\eta_i^2
        \le
        \frac{\alpha\eta_i}{3},
    \]
    and hence
    \[
        1-\frac{3\alpha\eta_i}{2}+2\beta^2\eta_i^2
        \le
        1-\alpha\eta_i .
    \]
    Taking expectations and absorbing constants gives
    \[
        \delta_{i+1}
        \le
        (1-\alpha\eta_i)\delta_i
        +
        C\eta_i B(p)^2
        +
        C\eta_i^2\sigma^2 ,
    \]
    where $C>0$ depends only on $\alpha,\beta$ and the step-size constants.

    It remains to solve this recursion. Define
    \[
        \Lambda_i := \delta_i - K B(p)^2,
    \]
    where $K>0$ is chosen large enough such that, for all $i$,
    \[
        (1-\alpha\eta_i)K B(p)^2 + C\eta_i B(p)^2
        \le
        K B(p)^2.
    \]
    Equivalently, it is enough to take $K\ge C/\alpha$. Then
    \[
        \Lambda_{i+1}
        \le
        (1-\alpha\eta_i)\Lambda_i
        +
        C\eta_i^2\sigma^2 .
    \]
    With $\eta_i=c/(i+\gamma)$ and $\alpha c>1$, the standard
    recursion bound gives
    \[
        \Lambda_n
        =
        O\!\left(n^{-\alpha c}\right)
        +
        O\!\left(\frac{\sigma^2}{n}\right).
    \]
    Because $\alpha c>1$, the initialization term is $O(1/n)$. Therefore,
    \[
        \delta_n
        =
        \mathbb E\|w_n-w^\star\|^2
        =
        O\!\left(B(p)^2\right)
        +
        O\!\left(\frac{\sigma^2}{n}\right).
    \]
    If $\sigma^2=O(1)$, this becomes
    \[
        \mathbb E\|w_n-w^\star\|^2
        =
        O\!\left(B(p)^2\right)+O(1/n).
    \]
    The $O(B(p)^2)$ term is the limiting neighborhood induced by the systematic
    gradient bias.
\end{proof}

\paragraph{Plug-in for plain imputed SGD.}
Under independent hMCAR/sMAR, Proposition~\ref{prop:grad_bias_fo_struct} gives $b(p)=O(\|p\|)$. The variance of $\hat g^{(p)}$ is bounded by a constant under the standing $L^2$ assumptions. Hence
\begin{equation*}
    \EE\|w_n-w^\star\|^2\;=\;O(\|p\|^2)+O(1/n).
\end{equation*}

\paragraph{Plug-in for Richardson-SGD.}
Under independent hMCAR/sMAR, Proposition~\ref{prop:first-order-debias} gives $b(p)=O(\|p\|^2)$. Using~\eqref{eq:variance-bound}, $\sigma^2$ remains $O(1)$ in the moderate-$C$ regime. Substituting in~\eqref{eq:generic-onepass} yields
\begin{equation*}
    \EE\|w_n-w^\star\|^2\;=\;O(\|p\|^4)+O(1/n).
\end{equation*}
The analogous statement for the test loss follows by smoothness of $L$ around $w^\star$. The $k$-step variant gives $b(p)=O(\|p\|^{k+1})$ and a missingness floor $O(\|p\|^{2(k+1)})$.

\paragraph{Plug-in version with estimated mechanism.}
Combining Proposition~\ref{prop:plug-in} and~\eqref{eq:generic-onepass}, the one-pass bound becomes
\[
    \EE\|w_n-w^\star\|^2
    \;=\;
    O\!\left(\|p\|^4+\delta_p^2+\|p\|_\infty^2\delta_q^2\right)+O(1/n).
\]
Whenever $\delta_p,\delta_q$ shrink at rate $o(\|p\|^2)$, the exact-mechanism rate is recovered.

\paragraph{Multi-epoch behavior.}
The above analysis only covers one pass, else the imputed gradients seen on different epochs are not independent. The empirical study of Section~\ref{sec:experiments} indicates that Richardson remains effective in multi-epoch training; a formal multi-epoch analysis is left to future work.


\section{Implementation details}
\label{app:implementation-details}

This appendix describes the experimental protocol used in Section~\ref{sec:experiments} and in Appendix \ref{app:additional-experiments} below. All experiments are run with stochastic gradient descent for $5$ epochs, minibatch size $64$, average missingness level $\bar p=0.20$, and first-order Richardson scale $C=2$. Unless stated otherwise, all reported curves are averaged over repeated runs with the same protocol across methods.

\paragraph{Models.}
We consider three generalized linear models: linear regression with Gaussian noise, logistic regression for binary classification, and Poisson regression for count responses. All models are trained with an $\ell_2$ penalty. The regularization parameter is fixed to $\lambda=10^{-3}$ for every model family and dataset.

\paragraph{Missingness mechanisms.}
We evaluate three missingness mechanisms. The first is homogeneous MCAR, denoted \texttt{mcar}, where each entry is missing independently with the same probability $p$. The second is heterogeneous MCAR, denoted \texttt{hetero\_mcar}, where missingness probabilities are generated from row and column multipliers and then calibrated to have average missingness $\bar p$. Concretely, the unnormalized missingness scores are sampled uniformly in $[0,1]$ across covariates and rescaled so that their empirical mean equals $0.20$.

The third mechanism is scalable MAR, denoted \texttt{smar}. In this case, the oracle missingness intensity is
\[
    \lambda_{j}=p_j Q(U_{j}),
    \qquad
    Q(u)=\sigma(1.6u-0.3),
    \qquad
    U = a_{1,j}X_1 +  b_{2,j}X_2, \quad a_{1,\cdot}, b_{2, \cdot} \overset{\text{i.i.d}}{\sim}  \  \mathcal{U}(0,1)\ .
\]
where $\sigma$ is the logistic sigmoid. The coordinate-specific constants $p_j$ are calibrated so that the average missingness is $\bar p=0.20$. This is the same scalable MAR mechanism as in Section~\ref{sec:setting}.

\paragraph{Methods compared.}
We compare the complete-data baseline, plain imputation-based SGD, and Richardson-corrected SGD. The complete-data baseline, denoted \texttt{No missing vals}, is trained on the clean unmasked training data. The plain imputation baselines are zero imputation, mean imputation, $k$-nearest-neighbor imputation, MICE, and MICE with random-forest base learners, denoted respectively by \texttt{Zero}, \texttt{Mean}, \texttt{KNN}, \texttt{MICE}, and \texttt{MICE+RF}. The corresponding Richardson variants are denoted \texttt{Rich.--Zero}, \texttt{Rich.--Mean}, \texttt{Rich.--KNN}, \texttt{Rich.--MICE}, and \texttt{Rich.--MICE+RF}. All imputers are taken with default parameters from \texttt{scikit-learn}. The experiments are repeated 30 times with different seeds, for the training of SGD methods, and averaged results, along their standard deviations, are displayed.

\paragraph{Metric.}
The main metric is the parameter mean-squared error
\[
    \mathrm{MSE}_{w}(t)
    =
    \frac{1}{d_w}
    \bigl\|\hat w_t-w^\star\bigr\|_2^2,
\]
where $d_w$ is the parameter dimension, $\hat w_t$ is the SGD iterate after epoch $t$, and $w^\star$ is the complete-data reference parameter described below. The metric is reported once per epoch for $5$ epochs. For real datasets, $w^\star$ denotes the minimizer of the complete-data ridge $10^{-3}$ penalized empirical, not a population ground truth (see the paragraph \textit{Reference parameter} below).

\paragraph{Learning-rate calibration.}
The optimization geometry varies substantially across model families and datasets. To avoid confounding imputation effects with poorly tuned learning rates, we calibrate the initial learning rate $\eta_0$ separately for each pair of model family and dataset.

For each pair, we first take the family-level default learning rate $\eta_0 = 10^{-2}$. We then evaluate the geometric grid
\[
    \eta_0
    \in
    \eta_0^{\mathrm{def}}
    \cdot
    \left\{
    \frac14,\frac12,1,2,4
    \right\}.
\]
For every candidate, we run SGD, without missing data, on the standardized training fold using the same number of epochs, minibatch size, and regularization parameter as in the missing-data experiments. We select the learning rate that minimizes the final iteration parameter MSE,
\[
    \frac{1}{d_w}
    \bigl\|\hat w_T-w^\star\bigr\|_2^2\ .
\]
This calibration is performed without missingness and without imputation. The selected learning rate is then fixed and reused for all imputation methods, Richardson variants, and missingness mechanisms for that model--dataset pair. Thus, comparisons between MCAR, heterogeneous MCAR, and sMAR within the same row use the same calibrated $\eta_0$.

\paragraph{Dataset budget and preprocessing.}
Each dataset uses $2{,}000$ training samples. Real datasets with fewer observations are bootstrapped to this size when needed. Test sets contain $1{,}000$ samples. Covariates are standardized columnwise on the training fold and the same transformation is applied to the test fold.

The response variable is rescaled depending on the model family. For linear regression on real datasets, the response is z-scored on the loaded sample. This keeps the scale of $w^\star$ comparable across synthetic and real datasets; in particular, using raw elevation in the Covertype regression task produces parameters much larger than those in the synthetic linear experiments. Synthetic linear responses are left unchanged, since the data-generating process already gives a comparable response scale.

For logistic regression, the response is binary and no rescaling is applied. For Poisson regression, real count responses are rescaled to have mean approximately $2$ and then rounded to integer counts, corresponding to a log-mean near $0.7$.

\paragraph{Reference parameter.}
For synthetic datasets, $w^\star$ is the parameter used in the data-generating process. The covariates are generated to be approximately centered and standardized, so the train-fold standardization is nearly idempotent and the generating parameter remains the appropriate reference.
For real datasets, there is no closed-form ground-truth parameter. We therefore compute $w^\star$ by L-BFGS-B optimization of the exact ridge penalized empirical loss on the complete, standardized training data, using the same ridge regularization parameter $\lambda=10^{-3}$ as in the SGD runs. This gives the complete-data regularized minimizer of the observed sample loss and serves as the reference parameter for the reported MSE.

\paragraph{Datasets.}
The datasets used in the experiments are listed in Table~\ref{tab:datasets-implementation}. Synthetic datasets are generated with Gaussian covariates. Real datasets are taken from standard \texttt{scikit-learn} or OpenML sources and transformed as indicated.
\begin{table}[h]
    \centering
    \caption{Datasets used in the experiments.}
    \small
    \label{tab:datasets-implementation}
    \begin{tabular}{llll}
        \toprule
        Family & \# & Dataset            & Source and preprocessing                                                    \\
        \midrule
        Linear
               & 1  & Synth-A            & Synthetic, $10$D iid Gaussian                                               \\
        Linear
               & 2  & Synth-B            & Synthetic, $15$D AR-style covariance $\Sigma_{jk}=0.9^{|j-k|}$              \\
        Linear
               & 3  & Diabetes           & Real, bootstrapped, z-scored response                                       \\
        Linear
               & 4  & Covertype-reg      & Real, z-scored elevation from $9$ continuous features                       \\
        \midrule
        Logistic
               & 1  & Synth-A            & Synthetic, $10$D iid Gaussian                                               \\
        Logistic
               & 2  & Breast cancer      & Real, bootstrapped to $2{,}000$ samples                                     \\
        Logistic
               & 3  & Covertype          & Real, class $1$ versus all, $10$ continuous features                        \\
        Logistic
               & 4  & California housing & Real; binary response: house price >  median price                    \\
        \midrule
        Poisson
               & 1  & Synth-A            & Synthetic, $10$D iid Gaussian                                               \\
        Poisson
               & 2  & Synth-B            & Synthetic, $8$D iid Gaussian                                                \\
        Poisson
               & 3  & Bike sharing       & Real, hourly rental count from numeric features                             \\
        \bottomrule
    \end{tabular}
\end{table}
For the Bike Sharing Demand dataset, we use the hourly rental count as the response and retain eight numeric features: year, month, hour, weekday, temperature, feeling temperature, humidity, and windspeed.

\section{Robustness to errors in the estimated missingness mechanism}
\label{app:exp-robust}

The previous experiments kept $p$ and $q$ known as oracles. In practice, however, these quantities need to be estimated, and Proposition~\ref{prop:plug-in} provides an upper bound of the error induced by such estimations.  In the following experiment, we test the impact of  estimating $p$ and $q$ on Richardson-SGD, on top of imputation by zero, for logistic regression in the hMCAR setting. We perturb the estimated mechanism $(\hat p,\hat q)$ by additive noise with magnitudes $\delta_p,\delta_q$ and report the parameter MSE as a function of $(\delta_p,\delta_q)$. Table~\ref{tab:delta_results} shows the robustness of Richardson to plug-in estimation. With a reference parameter MSE of $2.647 \times 10^{-2}$ for no missing data and $5.399 \times 10^{- 2}$ for plain imputation by zero, we see that even under high ratio mismatch, Richardson performs better than simple imputation. The only worse errors occur when $\delta_p = 0.3$ and $\delta_q \geq 0.2$, which we put in italics.

\begin{table}[H]
    \centering
    \caption{MSE across $\delta_q$ and $\delta_p$; all entries are $\times 10^{-2}$.}
    \label{tab:delta_results}
    \setlength{\tabcolsep}{5.5pt}
    \renewcommand{\arraystretch}{1.12}
    \begin{tabular}{c cccccc}
        \toprule
        $\delta_q \backslash \delta_p$
             & 0.00 & 0.05 & 0.10 & 0.15 & 0.20 & 0.30          \\
        \midrule
        0.00 & 3.27 & 3.53 & 3.75 & 3.97 & 4.24 & 4.83          \\
        0.05 & 3.60 & 3.83 & 4.14 & 4.40 & 4.67 & 5.25          \\
        0.10 & 3.70 & 3.89 & 4.13 & 4.40 & 4.57 & 5.05          \\
        0.15 & 3.96 & 4.17 & 4.37 & 4.66 & 4.89 & 5.31          \\
        0.20 & 4.34 & 4.51 & 4.73 & 4.93 & 5.12 & \textit{5.49} \\
        0.30 & 4.62 & 4.77 & 4.89 & 4.99 & 5.13 & \textit{5.40} \\
        \bottomrule
    \end{tabular}
\end{table}

\section{Additional GLM experiments}
\label{app:additional-experiments}

We provide additional comparisons of SGD with several imputation rules, with and without Richardson, across linear (Gaussian), logistic, and Poisson regression on synthetic and real datasets. Each figure shows the parameter-MSE trajectory under SGD with and without Richardson on top of several imputation schemes; the bias formulas of App.~\ref{app:glm_bias_formulas} predict the per-model behavior. We organize the figures by GLM family $\times$ missingness mechanism. The key empirical observations are as follows.

\textbf{Key empirical observations.}
\begin{itemize}
    \item Richardson is consistently effective on top of mean, MICE, MICE-RF, and $k$-NN imputation across the three GLMs and the three mechanisms.
    \item Although the theory only covers one-pass SGD, Richardson remains robust empirically in multi-epoch training.
    \item In several settings, the gains are most pronounced in the first epoch, in agreement with the one-pass theory of Section~\ref{subsec:one-pass}.
\end{itemize}

\newpage
\subsection{Linear (Gaussian) regression}

\begin{figure}[H]
    \centering
    \includegraphics[width=1\textwidth]{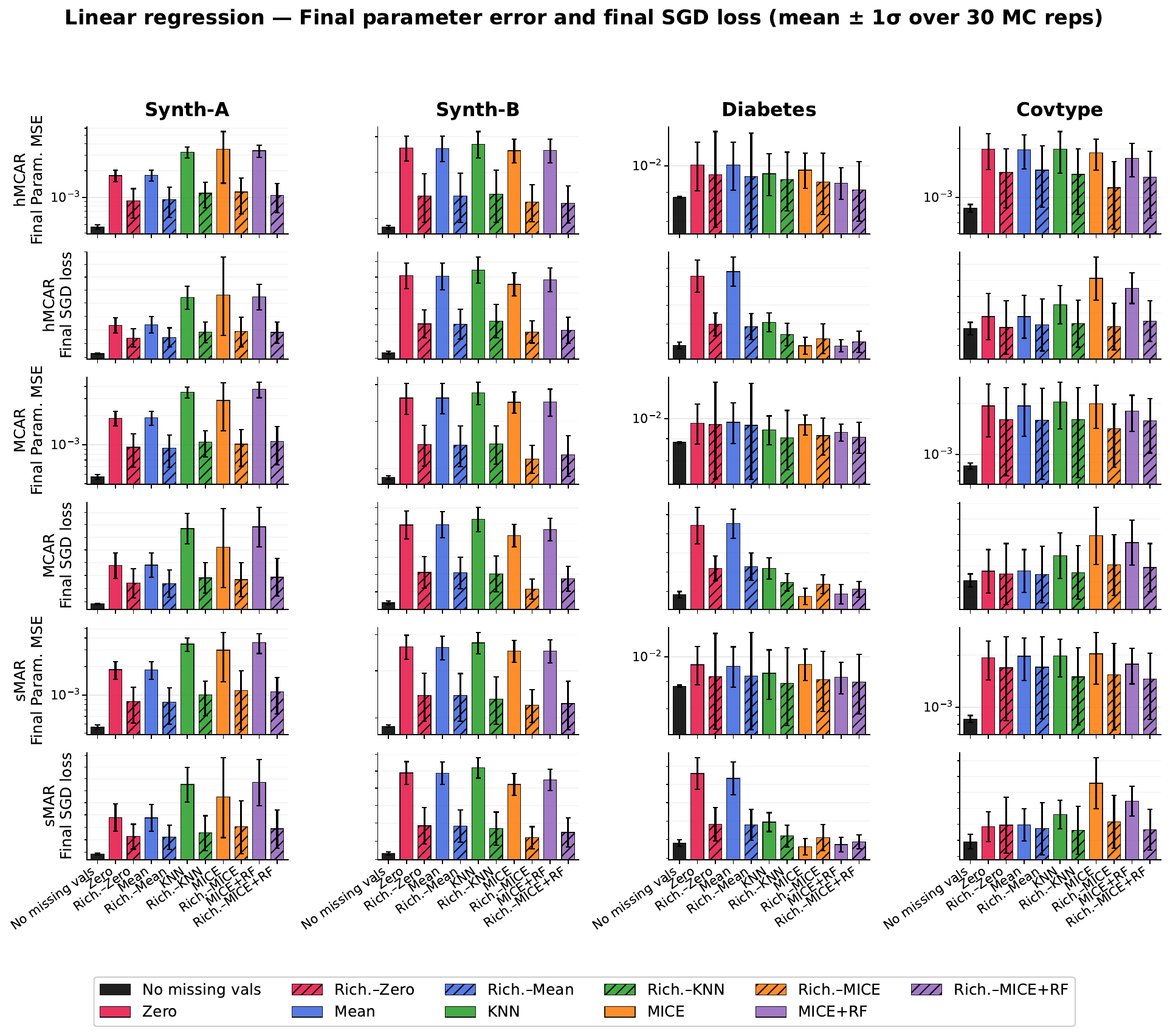}
    \caption{Final parameter MSE and test loss for linear regression, under MCAR, heterogeneous MCAR and sMAR mechanisms, on four different datasets.}
    \label{fig:glm-linear}
\end{figure}

\begin{figure}[H]
    \centering
    \includegraphics[width=1\textwidth]{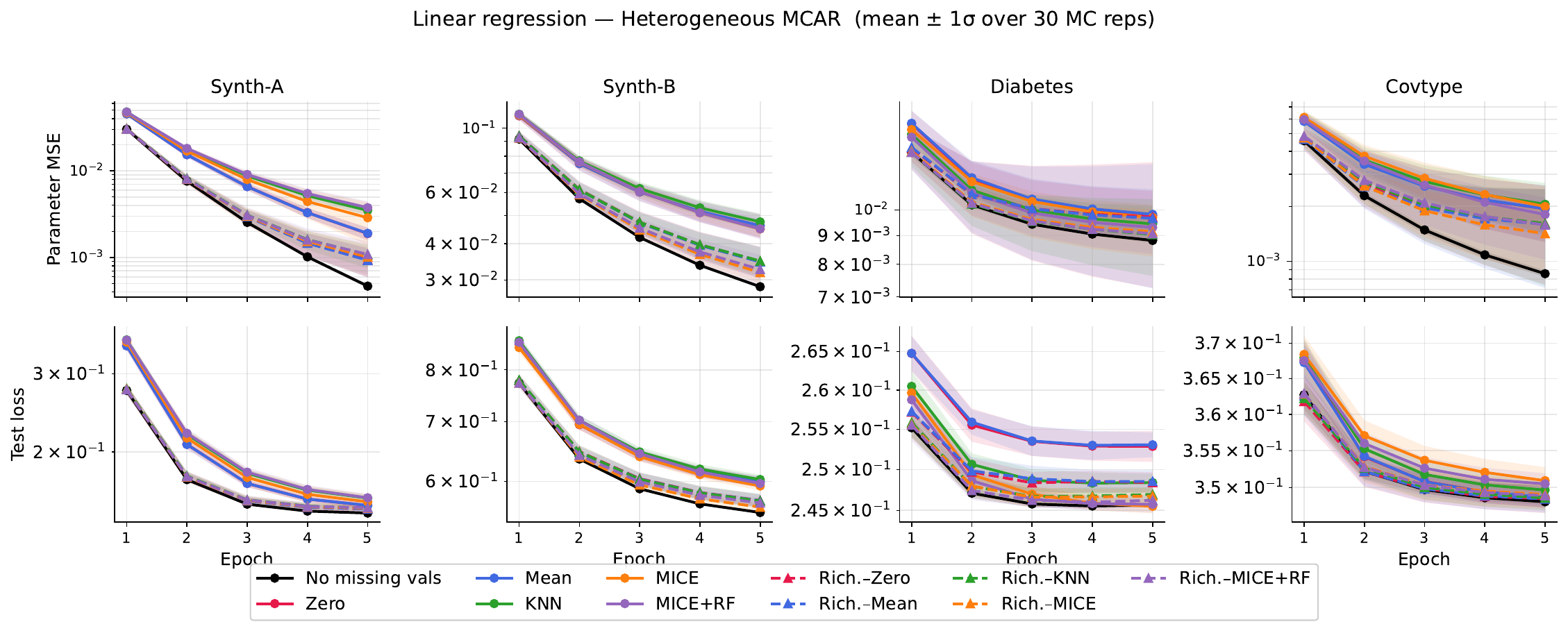}
    \caption{Convergence rate for parameter MSE and test loss for linear regression, under heterogeneous MCAR and sMAR mechanisms, on four different datasets.}
    \label{fig:cv-linear}
\end{figure}

\newpage
\subsection{Logistic regression}

\begin{figure}[H]
    \centering
    \includegraphics[width=1\textwidth]{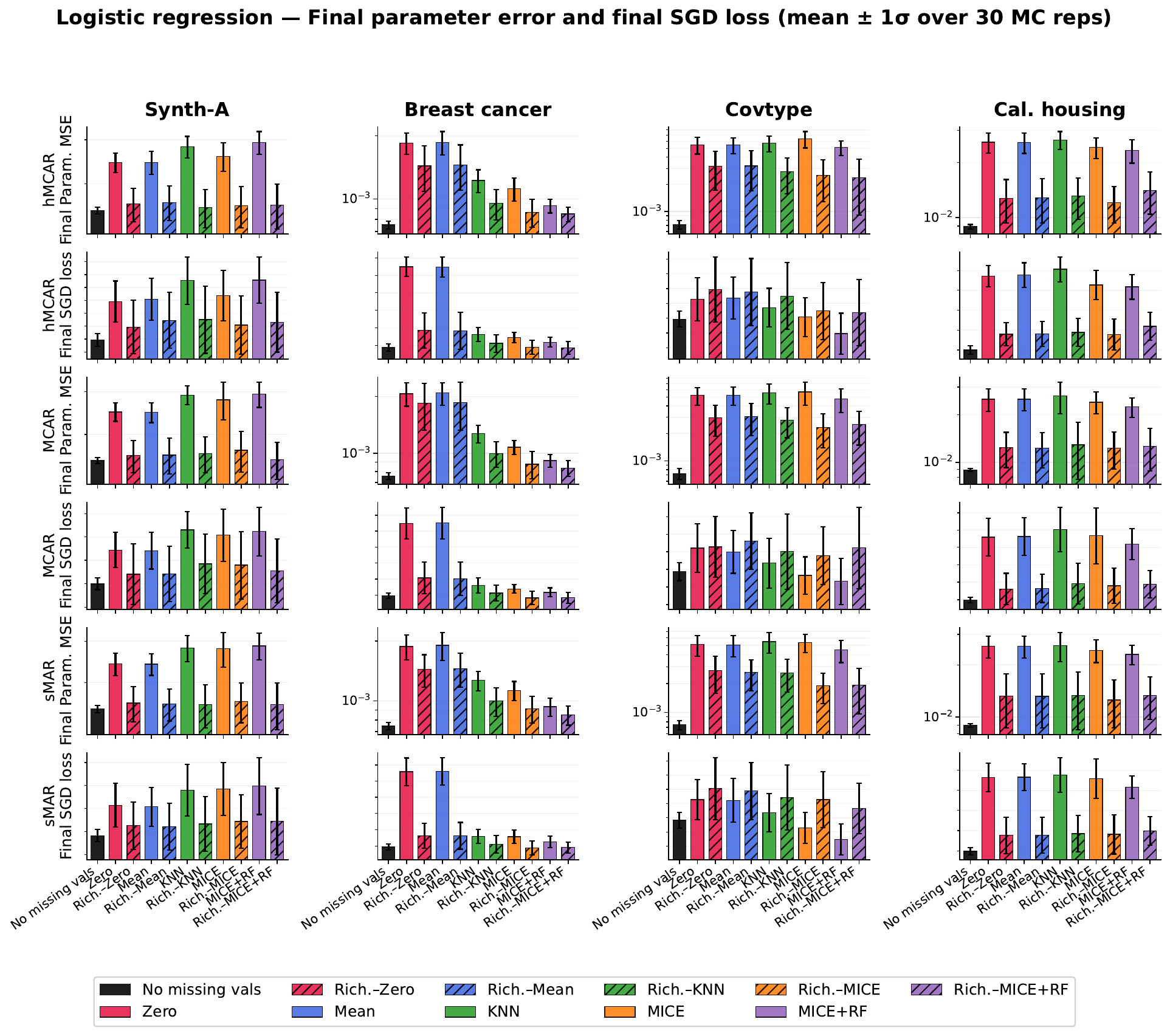}
    \caption{Final parameter MSE and test loss for logistic regression, under MCAR, heterogeneous MCAR and sMAR mechanisms, on four different datasets.}
    \label{fig:glm-logistic}
\end{figure}

\begin{figure}[H]
    \centering
    \includegraphics[width=1\textwidth]{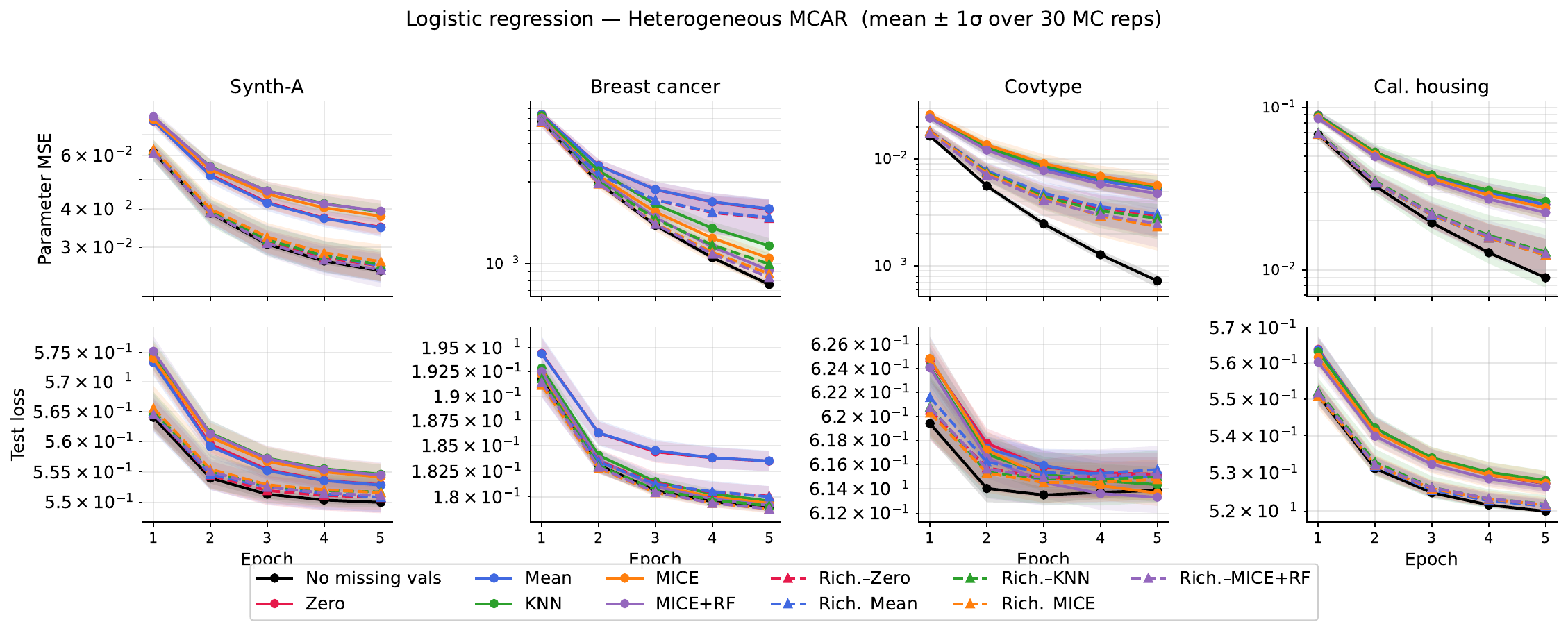}
    \caption{Convergence rate for parameter MSE and test loss for logistic regression, under heterogeneous MCAR and sMAR mechanisms, on four different datasets.}
    \label{fig:cv-log}
\end{figure}

\newpage
\subsection{Poisson regression}

\begin{figure}[H]
    \centering
    \includegraphics[width=0.9\textwidth]{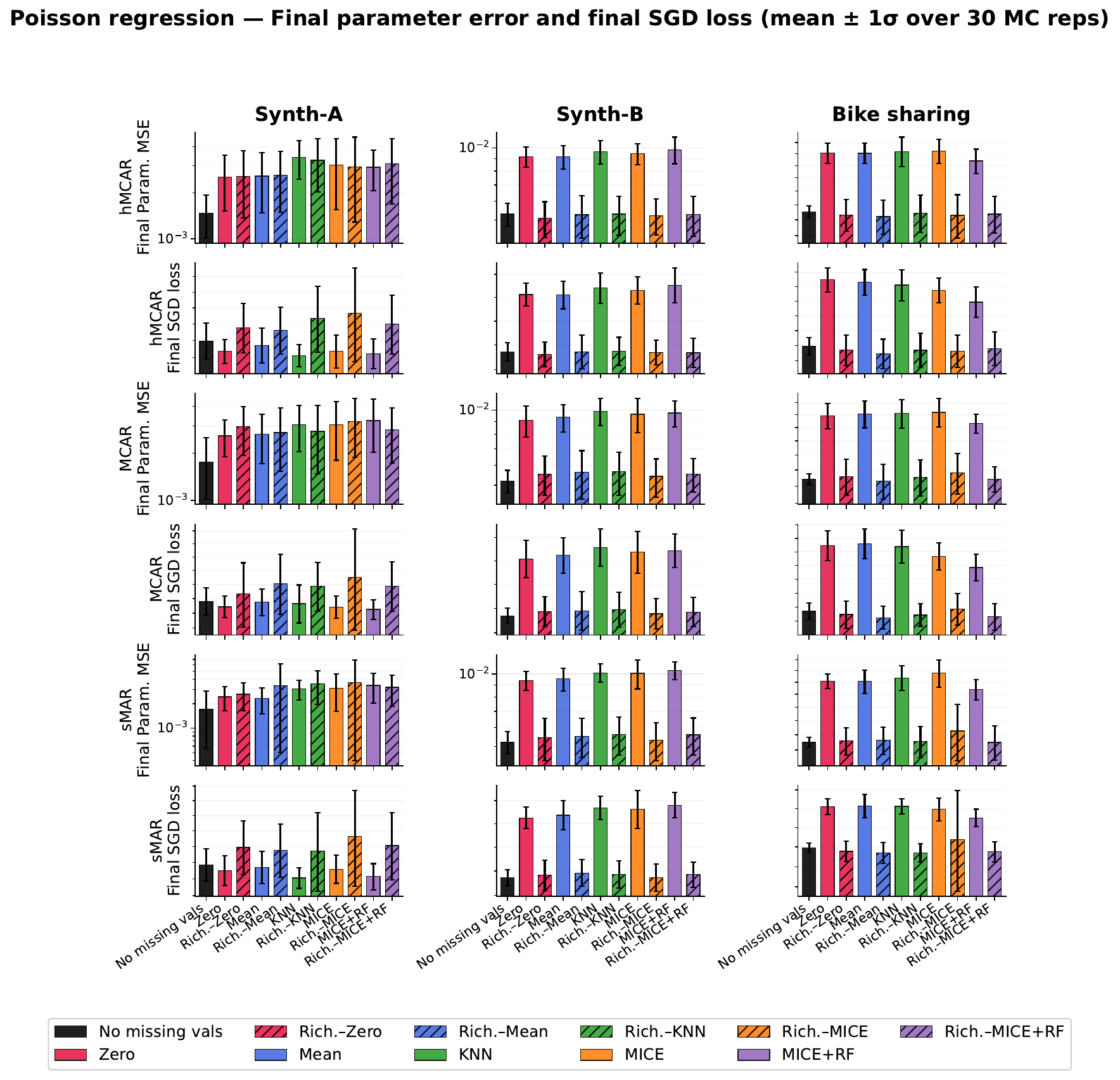}
    \caption{Final parameter MSE and test loss for Poisson regression, under MCAR, heterogeneous MCAR and sMAR mechanisms, on four different datasets.}
    \label{fig:glm-poisson}
\end{figure}

\begin{figure}[H]
    \centering
    \includegraphics[width=0.8\textwidth]{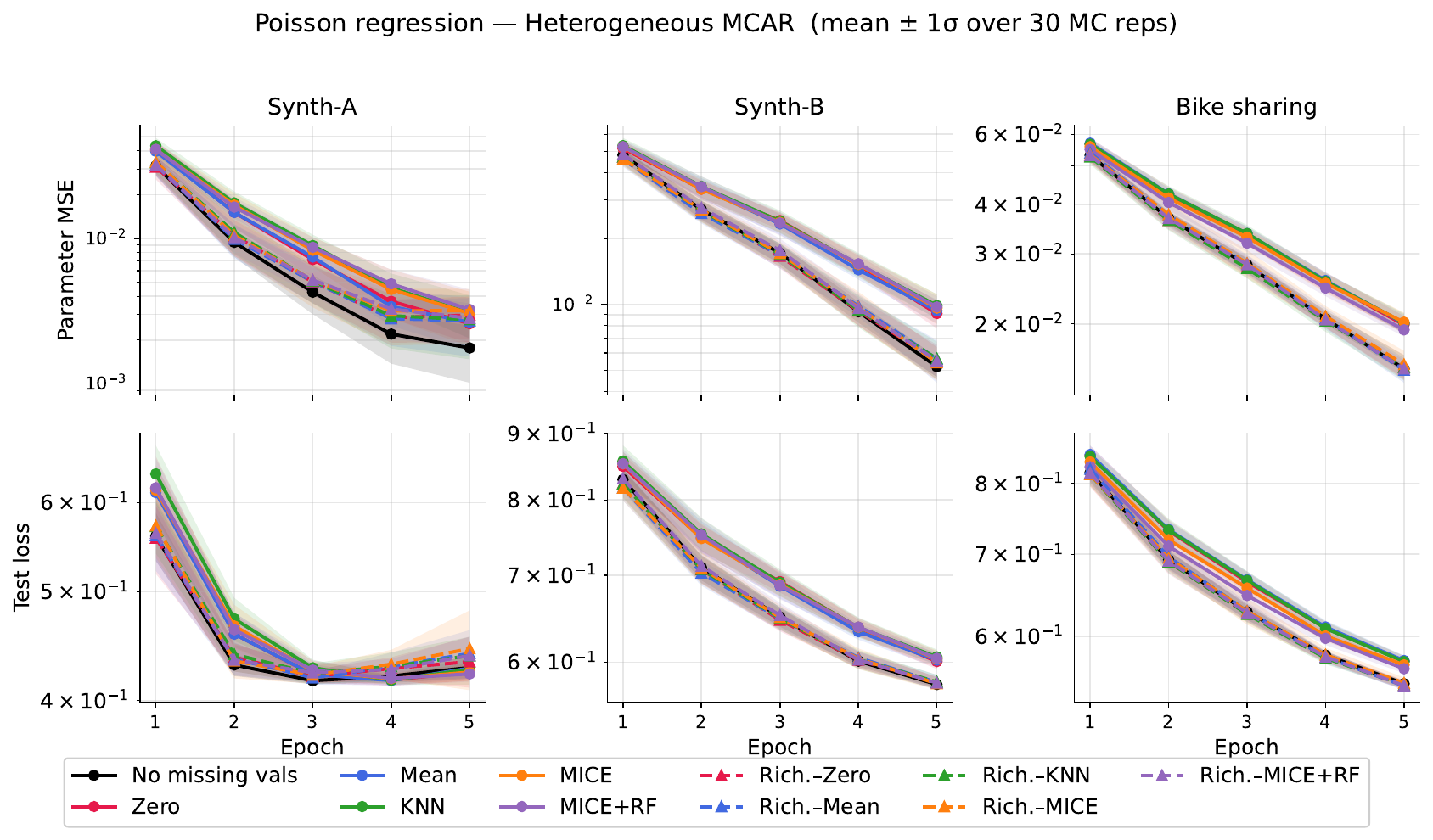}
    \caption{Convergence rate for parameter MSE and test loss for Poisson regression, under heterogeneous MCAR and sMAR mechanisms, on four different datasets.}
    \label{fig:cv-poisson}
\end{figure}

\section{Comparison with other schemes}

The experiments above show that Richardson-SGD improves performance across the three GLMs considered and across all imputation techniques tested. This is the regime for which the method is primarily intended: when the learner has access to an imputation pipeline, but does not want to impose a model-specific correction or strong structural assumptions on the covariate distribution. In this sense, Richardson-SGD is best viewed as a model-agnostic and distribution-agnostic debiasing layer on top of imputation, rather than as a competitor to specialized estimators designed for a particular statistical model. Consequently, the most direct comparison is with the same imputation pipeline used without Richardson.

For completeness, we nevertheless include a more specialized benchmark in the linear-regression setting. In this case, tailored alternatives are available: debiased SGD under hMCAR \citep{sportisse2020debiasing}, and SAEM-type methods under parametric assumptions on the covariate distribution \citep{jiang2020logistic}. These methods are designed specifically for this setting, so Richardson-SGD is not expected to dominate them. The point of the comparison is instead to test whether a generic Richardson correction remains competitive even in a regime where model-specific methods have an intrinsic advantage.

The results support this interpretation. Richardson-SGD performs close to debiased SGD and improves over SAEM, especially on non-synthetic datasets where the parametric assumptions underlying SAEM are less well matched to the data. The main failure case occurs for Richardson combined with MICE in the uncorrelated Gaussian-covariate setting. This behavior is expected: when covariates are independent, MICE has essentially no cross-feature signal to exploit and may fit noise, making it a poor base imputer. In such cases, Richardson inherits the limitations of the underlying imputation procedure.

\begin{table}[ht]
    \centering
    \caption{Parameter MSE and runtime results. Values are reported as mean $\pm$ standard deviation. SAEM is run only once, with a runtime of 5 s for Synth-A and Synth-B, and 100 s otherwise.}
    \begin{tabular}{llcc}
        \toprule
        Dataset      & Method                          & PMSE                                           & Time (s) \\
        \midrule
        Synth-A      & No missing data (ref)           & $5.51{\times}10^{-5} \pm 1.10{\times}10^{-6}$  & $0.019$  \\
        Synth-A      & Zero-Impute                     & $4.11{\times}10^{-3} \pm 1.00{\times}10^{-3}$  & $0.021$  \\
        Synth-A      & Rich. -- Zero                   & $2.36{\times}10^{-4} \pm 3.70{\times}10^{-5}$  & $0.028$  \\
        Synth-A      & Rich. -- MICE                   & $1.56{\times}10^{-2} \pm 9.60{\times}10^{-3}$  & $0.036$  \\
        Synth-A      & Debiased SGD (Sportisse et al.) & $2.99{\times}10^{-4} \pm 1.10{\times}10^{-4}$  & $0.057$  \\
        Synth-A      & SAEM                            & $6.00{\times}10^{-4} \pm 0$                    & $5.00$   \\
        \midrule
        Synth-B      & No missing data (ref)           & $2.39{\times}10^{-3} \pm 1.50{\times}10^{-5}$  & $0.029$  \\
        Synth-B      & Zero-Impute                     & $4.95{\times}10^{-1} \pm 3.22{\times}10^{-2}$  & $0.031$  \\
        Synth-B      & Rich. -- Zero                   & $4.97{\times}10^{-2} \pm 1.80{\times}10^{-3}$  & $0.055$  \\
        Synth-B      & Rich. -- MICE                   & $3.42{\times}10^{-2} \pm 4.00{\times}10^{-3}$  & $0.057$  \\
        Synth-B      & Debiased SGD (Sportisse et al.) & $4.64{\times}10^{-3} \pm 2.00{\times}10^{-3}$  & $0.073$  \\
        Synth-B      & SAEM                            & $2.33{\times}10^{-2} \pm 0$                    & $5.00$   \\
        \midrule
        Covtype      & No missing data (ref)           & $1.94{\times}10^{-2} \pm 0$                    & $0.028$  \\
        Covtype      & Zero-Impute                     & $8.04{\times}10^{-2} \pm 5.05{\times}10^{-2}$  & $0.033$  \\
        Covtype      & Rich. -- Zero                   & $2.60{\times}10^{-2} \pm 3.55{\times}10^{-3}$  & $0.055$  \\
        Covtype      & Rich. -- MICE                   & $2.37{\times}10^{-2} \pm 1.50{\times}10^{-3}$  & $0.057$  \\
        Covtype      & Debiased SGD (Sportisse et al.) & $1.99{\times}10^{-2} \pm 5.50{\times}10^{-3}$  & $0.079$  \\
        Covtype      & SAEM                            & $2.58{\times}10^{-1} \pm 0$                    & $100.00$ \\
        \midrule
        Cal. housing & No missing data (ref)           & $2.62{\times}10^{-2} \pm 5.00{\times}10^{-5}$  & $0.026$  \\
        Cal. housing & Zero-Impute                     & $1.02{\times}10^{-1} \pm  9.76{\times}10^{-3}$ & $0.021$  \\
        Cal. housing & Rich. -- Zero                   & $7.15{\times}10^{-2} \pm 1.23{\times}10^{-2}$  & $0.051$  \\
        Cal. housing & Rich. -- MICE                   & $5.10{\times}10^{-2} \pm 1.49{\times}10^{-2}$  & $0.049$  \\
        Cal. housing & Debiased SGD (Sportisse et al.) & $3.16{\times}10^{-2} \pm 9.55{\times}10^{-3}$  & $0.065$  \\
        Cal. housing & SAEM                            & $2.70{\times}10^{-1} \pm 0$                    & $100.00$ \\
        \bottomrule
    \end{tabular}
\end{table}

\section{Robustness to misspecification of the missingness mechanism}
\label{app:exp-misspecification}

This appendix tests Richardson-SGD under misspecification of the thinning mechanism. Missing values are generated under the sMAR mechanism of Appendix~\ref{app:implementation-details}, where missingness depends on $X_1$ and $X_2$, but Richardson thinning uses the hMCAR approximation $p_j$ instead of the true conditional probabilities $p_jq_j(V)$. We run linear regression on the four datasets in Table~\ref{tab:datasets-implementation}, a setting where the zero-imputation gradient bias is a polynomial of degree at most two in the missingness probabilities; see Appendix~\ref{app:linreg}. We compare zero imputation and MICE, with and without first-order Richardson correction, over $10$ runs. Figure~\ref{fig:misspecification-convergence} shows the parameter-MSE and test-loss trajectories, and Table~\ref{tab:section16-summary} reports final values.

Richardson remains robust to this misspecification. With zero imputation, Richardson improves over plain zero imputation on all datasets, often nearly matching the complete-data baseline. With MICE, Richardson improves performance on Synth-A, Synth-B, and Diabetes. The main exception is California Housing, where Richardson--MICE becomes unstable near the last epoch; the plotted variance is capped for readability. This instability is consistent with the variance amplification of Richardson extrapolation discussed in Section~\ref{subsec:variance}, and may also reflect an imperfect learning-rate choice.

Overall, treating sMAR data as hMCAR does not eliminate the benefit of Richardson-SGD in these experiments, especially with zero imputation. However, the California Housing--MICE case shows that misspecification can interact with the imputation rule and optimization dynamics.

\begin{figure}[H]
    \centering
    \includegraphics[width=1\linewidth]{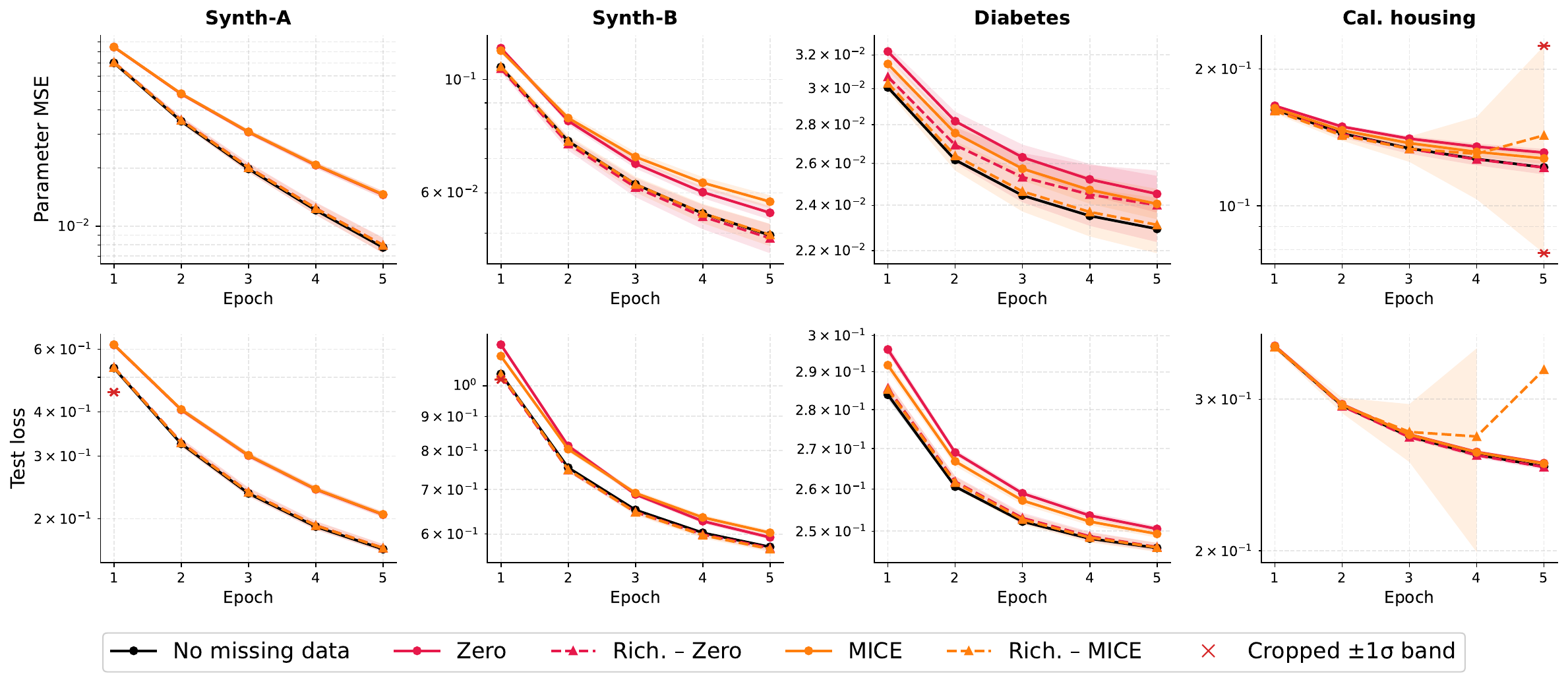}
    \caption{\small Richardson-SGD under misspecification of the missingness mechanism. Missing values are sMAR, but Richardson-SGD is computed using an hMCAR approximation based only on marginal missingness probabilities.}
    \label{fig:misspecification-convergence}
\end{figure}

\begin{table}[ht]
    \centering
    \scriptsize
    \caption{\small Final parameter MSE and test loss for Richardson-SGD under misspecification of the missingness mechanism. The true mechanism is sMAR, while Richardson thinning uses an hMCAR approximation. Test losses are reported in units of $10^{-1}$.}
    \label{tab:section16-summary}
    \setlength{\tabcolsep}{3.5pt}
    \renewcommand{\arraystretch}{1.08}
    \begin{tabular}{lcccccccc}
        \toprule
         & \multicolumn{2}{c}{Synth-A}
         & \multicolumn{2}{c}{Synth-B}
         & \multicolumn{2}{c}{Diabetes}
         & \multicolumn{2}{c}{Cal. housing}                 \\
        \cmidrule(lr){2-3}\cmidrule(lr){4-5}\cmidrule(lr){6-7}\cmidrule(lr){8-9}
        Method
         & PMSE                                    & Loss
         & PMSE                                    & Loss
         & PMSE                                    & Loss
         & PMSE                                    & Loss   \\
        \midrule
        No missing data
         & $7.77{\scriptstyle\pm0.045}\,10^{-3}$   & $1.64$
         & $4.96{\scriptstyle\pm0.016}\,10^{-2}$   & $5.74$
         & $2.29{\scriptstyle\pm0.0037}\,10^{-2}$  & $2.46$
         & $1.21{\scriptstyle\pm0.00099}\,10^{-1}$ & $2.51$ \\

        Zero
         & $1.45{\scriptstyle\pm0.063}\,10^{-2}$   & $2.05$
         & $5.48{\scriptstyle\pm0.17}\,10^{-2}$    & $5.93$
         & $2.45{\scriptstyle\pm0.086}\,10^{-2}$   & $2.51$
         & $1.31{\scriptstyle\pm0.024}\,10^{-1}$   & $2.53$ \\

        Rich.--Zero
         & $7.99{\scriptstyle\pm0.75}\,10^{-3}$    & $1.65$
         & $4.89{\scriptstyle\pm0.32}\,10^{-2}$    & $5.72$
         & $2.40{\scriptstyle\pm0.16}\,10^{-2}$    & $2.46$
         & $1.21{\scriptstyle\pm0.041}\,10^{-1}$   & $2.50$ \\

        MICE
         & $1.46{\scriptstyle\pm0.061}\,10^{-2}$   & $2.06$
         & $5.77{\scriptstyle\pm0.17}\,10^{-2}$    & $6.03$
         & $2.41{\scriptstyle\pm0.070}\,10^{-2}$   & $2.49$
         & $1.27{\scriptstyle\pm0.035}\,10^{-1}$   & $2.52$ \\

        Rich.--MICE
         & $7.97{\scriptstyle\pm0.71}\,10^{-3}$    & $1.65$
         & $4.97{\scriptstyle\pm0.25}\,10^{-2}$    & $5.70$
         & $2.31{\scriptstyle\pm0.12}\,10^{-2}$    & $2.46$
         & $1.43{\scriptstyle\pm1.20}\,10^{-1}$    & $3.25$ \\
        \bottomrule
    \end{tabular}
\end{table}

\section{Why the two missingness scales must share the same imputation}
\label{app:linked-imputation}

Richardson correction compares two gradients evaluated at missingness scales $p$ and $Cp$. For the linear term to cancel, these two gradients must be generated by the \emph{same imputation operator}. In particular, entries that are missing at both scales must receive the same imputed value. This is why, in Section~\ref{sec:richardson}, we impute only once at the higher missingness scale $Cp$, and then restore the entries that were artificially hidden to obtain the lower-scale covariate.

We formalize this point. Let $\mathcal I$ be a data-independent imputation rule and define
\[
    \hat g_{\mathcal I}^{(p)}(w)
    :=
    g\!\left(w;\tilde X_{\mathcal I}^{(p)},Y\right),
    \qquad
    \Bbias_{\mathcal I}(w,p)
    :=
    \EE\!\left[\hat g_{\mathcal I}^{(p)}(w)\right]-\nabla L(w).
\]
By Proposition~\ref{prop:grad_bias_fo_struct},
\[
    \Bbias_{\mathcal I}(w,p)
    =
    \Aop_{\mathcal I}(w)p+\Rrem_{\mathcal I}(w,p),
    \qquad
    \Rrem_{\mathcal I}(w,p)=O(\|p\|^2)
\]
under independent hMCAR/sMAR. The first-order operator $\Aop_{\mathcal I}(w)$ depends on the imputation rule, since its $j$-th column is the expected gradient gap created by declaring coordinate $j$ missing.

If Richardson is applied with two possibly different imputation rules $\mathcal I_0$ and $\mathcal I_1$ at scales $p$ and $Cp$, respectively, then
\[
    \hat g^{\mathrm R}_{C,\mathcal I_0,\mathcal I_1}(w)
    :=
    \frac{
        C\,\hat g_{\mathcal I_0}^{(p)}(w)
        -
        \hat g_{\mathcal I_1}^{(Cp)}(w)
    }{C-1}.
\]
Its bias is
\[
    \begin{aligned}
        \EE\!\left[
                 \hat g^{\mathrm R}_{C,\mathcal I_0,\mathcal I_1}(w)
                 \right]
        -\nabla L(w)
         & =
        \frac{
            C\,\Bbias_{\mathcal I_0}(w,p)
            -
            \Bbias_{\mathcal I_1}(w,Cp)
        }{C-1} \\
         & =
        \tfrac{C}{C-1}
        \bigl(
        \Aop_{\mathcal I_0}(w)
        -
        \Aop_{\mathcal I_1}(w)
        \bigr)p
        +O(\|p\|^2).
    \end{aligned}
\]
Thus the $O(\|p\|)$ term cancels only if
\[
    \Aop_{\mathcal I_0}(w)=\Aop_{\mathcal I_1}(w).
\]
This condition is automatic when the two gradients are constructed from the same higher-scale imputation, as in Equation~(\ref{eq:richardson-grad}): common missing entries have identical imputed values at both scales, and the only difference between $\tilde X^{(p)}$ and $\tilde X^{(Cp)}$ comes from the entries artificially hidden by the thinning step.

By contrast, if one independently runs two stochastic imputers, for example two separate MICE chains at scales $p$ and $Cp$, then common missing entries may receive different imputations. The corresponding first-order operators need not coincide, so Richardson may leave an even bigger uncancelled $O(\|p\|)$ bias and can amplify stochastic imputation noise through the factor $(C-1)^{-1}$.

Figure~\ref{fig:linked-vs-unlinked-rich-mice} illustrates this effect on a linear-regression experiment on California Housing. The linked construction, which imputes once at scale $Cp$, $C = 1.5$, and restores artificially hidden entries, remains stable and improves over plain MICE. The unlinked construction, which runs independent MICE imputations at the two scales, loses the first-order cancelation and becomes unstable.

\begin{figure}[H]
    \centering
    \includegraphics[width=0.7\linewidth]{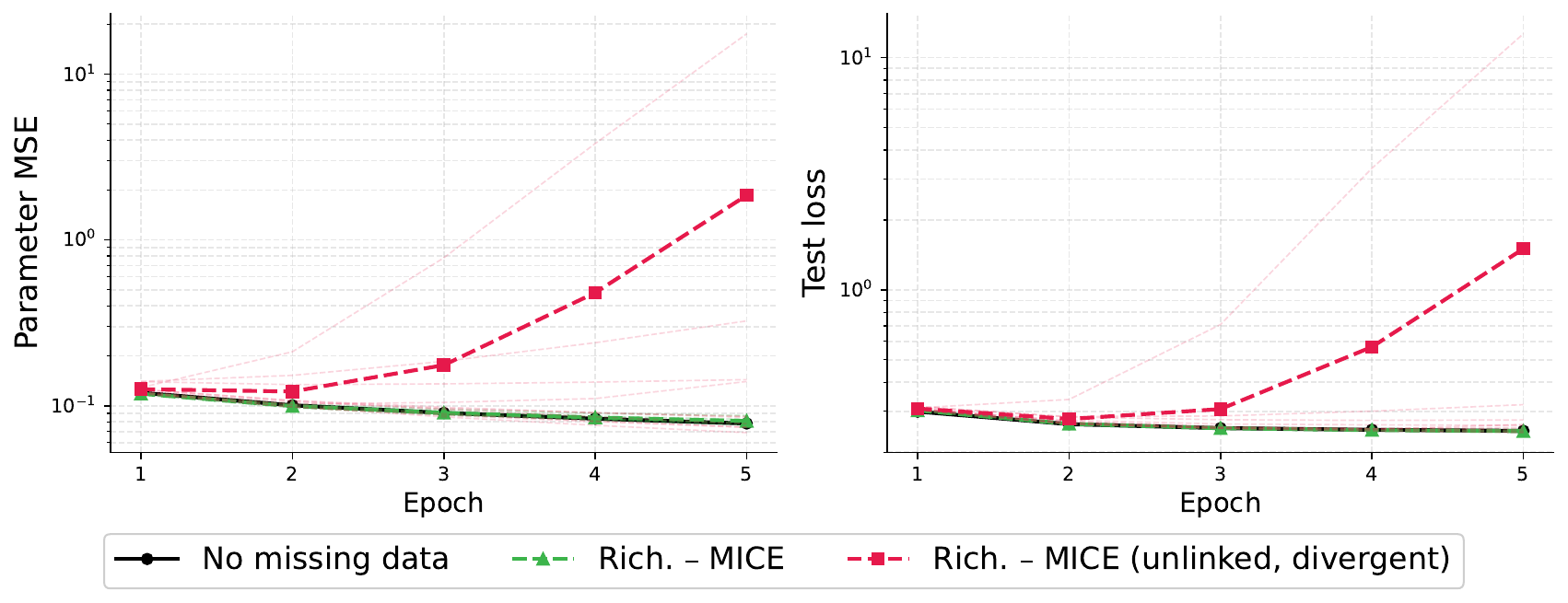}
    \caption{
        Linked versus unlinked Richardson--MICE on California Housing linear regression.
        The linked version uses one imputation at scale $Cp$ and restores the artificially hidden entries to obtain the scale-$p$ covariate.
        The unlinked version runs two separate MICE imputations at scales $p$ and $Cp$.
        Only the linked construction preserves the common imputed values required for first-order Richardson cancellation.
    }
    \label{fig:linked-vs-unlinked-rich-mice}
\end{figure}

\end{document}